\PassOptionsToPackage{sort&compress,numbers}{natbib}
\documentclass[preprint,12pt]{elsarticle}

\usepackage{mathtools}
\usepackage{amssymb}
\usepackage{booktabs}
\usepackage{tabularx} 
\usepackage{algorithm}
\usepackage{algpseudocode}
\usepackage{graphicx}
\usepackage{subcaption}
\usepackage{multirow}
\usepackage{placeins}
\usepackage{bm}
\providecommand{\symbfit}[1]{\bm{#1}}
\usepackage[colorlinks=true, linkcolor=blue, citecolor=blue, urlcolor=blue]{hyperref}
\usepackage{amsthm}

\theoremstyle{plain}
\newtheorem{theorem}{Theorem}

\theoremstyle{definition}

\theoremstyle{remark}

\usepackage{tikz}
\usetikzlibrary{arrows.meta,positioning,calc}

\usepackage{amsmath} 

\newcommand{\valci}[2]{\begin{tabular}[c]{@{}c@{}}#1\\[-1pt]$(\pm #2)$\end{tabular}}
\journal{}

\begin{document}

\begin{frontmatter}

\title{Nonlinear Data Integration via Kernel Methods for Data Collaboration Analysis}

\author[tsukuba-gsst]{Yamato Suetake}
\ead{s2620455@u.tsukuba.ac.jp}

\author[tsukuba-gsst]{Yuta Kawakami}

\author[tsukuba-iise]{Shunnosuke Ikeda}

\author[tsukuba-iise,tsukuba-tiar]{Yuichi Takano}

\affiliation[tsukuba-gsst]{
  organization={Graduate School of Science and Technology, University of Tsukuba},
  city={Tsukuba},
  state={Ibaraki},
  country={Japan}
}

\affiliation[tsukuba-iise]{
  organization={Institute of Systems and Information Engineering, University of Tsukuba},
  city={Tsukuba},
  state={Ibaraki},
  country={Japan}
}

\affiliation[tsukuba-tiar]{
  organization={Center for Artificial Intelligence Research, Tsukuba Institute for Advanced Research (TIAR), University of Tsukuba},
  addressline={1--1--1 Tennodai},
  city={Tsukuba-shi},
  postcode={305--8577},
  state={Ibaraki},
  country={Japan}
}

\begin{abstract}
Collaborative analysis of decentralized confidential datasets is important, but direct sharing of original datasets is often restricted by privacy and institutional constraints. Data collaboration (DC) analysis transforms each dataset into privacy-preserving intermediate representations via party-specific obfuscation functions and integrates them into common collaboration representations using an anchor dataset. However, many existing DC analysis methods rely on linear transformations for data obfuscation and integration, which may increase reconstruction risk. Although nonlinear dimensionality reduction can mitigate this risk, conventional linear integration methods cannot accurately align intermediate representations produced by nonlinear transformations. Moreover, existing integration methods mainly minimize discrepancies among parties and do not explicitly incorporate geometric or target-variable information useful for downstream analysis. To overcome these limitations, we first formulate linear target-normalized integration (LTI) as a linear integration method and then kernelize it to obtain kernel-based target-normalized integration (KTI). KTI admits a globally optimal solution via kernel ridge regression and an eigenvalue problem. We also introduce graph regularization and a centering constraint so that the target representation can capture geometric and target-variable information useful for downstream analysis. Experiments on image classification tasks demonstrate that KTI improves classification accuracy over existing linear integration methods under nonlinear dimensionality reduction, with further gains from target-variable-aware graph regularization and centering. The results also show that dimensionality reduction choices substantially affect both classification accuracy and reconstruction risk.
\end{abstract}

%


\end{frontmatter}

\section{Introduction}
In recent years, the collaborative use of decentralized datasets held by multiple parties has become increasingly important in domains such as healthcare, finance, manufacturing, and public administration.
Such collaboration can increase the effective sample size, cover more diverse data distributions, and improve prediction or estimation accuracy compared with analyses conducted by a single party.
In practice, however, original datasets often contain confidential information such as personal data and trade secrets, making central analysis---which aggregates all original datasets at a single analytical entity---legally, ethically, and practically challenging~\cite{need-protection1, need-protection2, GDPR}.
This has motivated privacy-preserving machine learning (PPML), which aims to enable collaborative analysis while protecting the confidentiality of original datasets~\cite{PPML}.

A wide variety of PPML frameworks have been proposed~\cite{DP, SecureComputation, HomomorphicEncryption1, FL1}.
Although these methods provide ways to analyze decentralized datasets without sharing original datasets, many tightly couple the privacy mechanism with a specific downstream analysis algorithm, restricting the reusability of the protected data.
Practical decentralized analysis therefore calls for a framework in which the privacy-preserving transformation performed by each party is independent of the downstream analysis performed by the analyst.

Data collaboration (DC) analysis has been proposed as such a framework~\cite{DCframework1, DCframework2}.
In DC analysis, each party independently applies its own obfuscation function to transform its original dataset into intermediate representations and provides only these representations to an analyst.
The analyst constructs integration functions using a shared anchor dataset and integrates the intermediate representations into collaboration representations in a common feature space.
Because each party's obfuscation process and the analyst's downstream analysis are completely decoupled, the resulting integrated dataset can be analyzed with any method without modifying the upstream obfuscation.

However, existing DC analysis methods still face important limitations in the design of integration functions.
First, existing linear integration methods do not guarantee a globally optimal solution without either retaining unnecessary degrees of freedom, imposing strong structural assumptions, or relying on rank conditions that may be violated.
Second, existing integration methods in DC analysis are based on linear integration functions, which cannot generally align intermediate representations produced by more secure nonlinear dimensionality reduction.
Third, existing integration methods do not explicitly incorporate information useful for downstream analysis into the construction of the target representation.

The main goal of this paper is to construct integration methods for DC analysis that retain an analytically tractable optimization structure while enabling the integration of intermediate representations obtained via nonlinear dimensionality reduction.
To this end, we first propose Linear Target-Normalized Integration (LTI), which formulates the construction of integration functions as a least-squares problem with an orthonormal target representation.
We show that LTI admits an analytically derived globally optimal target representation of the anchor dataset, which can be obtained by solving an eigenvalue problem without unnecessary degrees of freedom, strong structural assumptions, or rank conditions.

We then extend LTI to a nonlinear method based on kernel ridge regression, referred to as Kernel-based Target-Normalized Integration (KTI).
For a fixed target representation, each integration function in KTI is obtained from a finite-dimensional kernel ridge regression-type problem, and the target representation is estimated by solving an eigenvalue problem.
This extension enables integration functions to be constructed for intermediate representations obtained through nonlinear dimensionality reduction.

Furthermore, we incorporate graph-based regularization and a centering constraint into KTI.
Graph-based regularization enables geometric structure and target-variable information to be reflected in the target representation, whereas the centering constraint suppresses degenerate constant solutions.
Through numerical experiments on image classification tasks, we evaluate the proposed methods in terms of analytical accuracy, computation time, anchor dataset design, and reconstruction risk, and show that KTI improves prediction accuracy under nonlinear dimensionality reduction.

\section{Related Work}
\label{sec:related-work}

This section reviews studies related to privacy-preserving decentralized data analysis.
We first summarize representative PPML frameworks and data obfuscation methods, with particular attention to how they handle privacy protection, data reuse, and downstream analysis.
We then review DC analysis and existing integration methods, discussing their limitations and how they motivate the proposed method.

\subsection{Privacy-preserving Machine Learning and Data Obfuscation}

A wide variety of PPML techniques have been proposed to protect confidential information in decentralized data environments.
Representative approaches include differential privacy~\cite{DP_Algorithmic, DP, AGM}, secure computation based on cryptography~\cite{SecureComputation, Garbled-circuit1, Garbled-circuit2, SecureComputationExample2, SecureComputationExample3}, homomorphic encryption~\cite{HomomorphicEncryption1, HomomorphicEncryption2}, and federated learning~\cite{FL1, FL2, FLsurvey, FLSurvey2}.
Differential privacy controls the influence of individual records on analytical results, but stronger privacy guarantees often require larger amounts of noise and can degrade analytical accuracy and data utility~\cite{AGM}.
Secure computation and homomorphic encryption enable computation on encrypted data, but they incur substantially higher computational costs than plaintext computation~\cite{SecureComputation, HomomorphicEncryption1}.
Federated learning enables collaborative model training without directly sharing original datasets by exchanging model updates among parties.
However, it requires iterative communication, can increase communication costs and runtime, and may leak information through gradients or model updates~\cite{FLnon-iid, FLnon-iid2, FLChallenges, LeakageFromGradients, SecureAggregation}.
Furthermore, because aggregation in federated learning is tightly coupled with a specific learning algorithm, it is difficult to construct a reusable integrated dataset and then freely select a downstream analysis method.

Data obfuscation-based methods, also referred to as data perturbation, are closely related to this study because they transform original datasets before sharing and can decouple data aggregation from downstream analysis without iterative communication~\cite{PPDM, RandomProjection, rotation_perturbation}.
A representative example is multiplicative data perturbation using a common random matrix~\cite{common_obfuscation}.
However, when a shared linear transformation is employed, a structural vulnerability arises: if correspondences between original and transformed data are partially leaked, the transformation matrix can be estimated and other parties' datasets may be reconstructed~\cite{PPDM-attack}.
Private obfuscation mitigates this issue by allowing each party to apply an independently generated transformation~\cite{private_obfuscation}.
In this setting, however, party-specific transformations map datasets onto different bases or coordinate systems, and naive aggregation of the resulting intermediate representations tends to cause misalignment in the feature space.
These two limitations---reconstruction risk from shared transformations and misalignment from private transformations---indicate that privacy-preserving data obfuscation requires not only secure transformations of original datasets but also an integration mechanism that aligns party-specific intermediate representations.

\subsection{DC Analysis and Limitations of Existing Methods}

Data collaboration (DC) analysis integrates party-specific intermediate representations into collaboration representations using a shared anchor dataset~\cite{DCframework1, DCframework2}.
DC analysis provides theoretical privacy guarantees for intermediate representations~\cite{DCprivacy} and can be combined with differential privacy and federated learning to complement the limitations of each approach~\cite{DC-DP, FedDCL}.

Owing to its computational efficiency and versatility, DC analysis has been applied to a wide range of decentralized data analysis tasks, including predictive modeling and causal inference~\cite{DC-QE, DC-QE-med}, survival analysis~\cite{DC-COX}, explainable machine learning~\cite{DC-SHAP}, recommender systems~\cite{DCrecommender}, and clustering~\cite{DC-Clustering}.
Privacy-oriented variants of DC analysis have also been proposed to make correspondence with original datasets more difficult to establish~\cite{NRI-DC}.
These studies indicate that DC analysis is a practical framework for decentralized data analysis, but its performance and confidentiality depend strongly on how party-specific intermediate representations are integrated.

As a representative integration method in DC analysis, Imakura et al.~\cite{DCframework2} proposed the minimum perturbation problem (MPP), which minimizes discrepancies among the collaboration representations of the anchor dataset and has been widely adopted in subsequent DC studies~\cite{DC-DP, FedDCL, DC-QE, DC-QE-med, DCrecommender, DC-SHAP, DC-COX, DC-Clustering, NRI-DC, DC-fed-17, DCapp5}.
However, MPP retains an unnecessary invertible degree of freedom in the target representation, which can substantially affect downstream performance~\cite{DC-odc}.

Orthogonal data collaboration (ODC)~\cite{DC-odc} reduces this degree of freedom by restricting integration functions to orthogonal transformations and stabilizing the scale and inner-product structure of the collaboration representations.
However, ODC depends on a prescribed target representation and imposes a strong structural assumption on the obfuscation functions.

Kawakami et al.~\cite{KawakamiDC} proposed a generalized-eigenvalue-problem-based integration method that directly minimizes discrepancies among anchor collaboration representations without prescribing a target representation.
Although this method avoids both the invertible degree of freedom in MPP and the orthogonality assumption in ODC, it relies on rank conditions and lacks an explicit global optimality guarantee.
These formulations and their limitations are analyzed mathematically in Section~\ref{subsec:existing-linear-integration}.

Existing DC analysis methods often employ linear dimensionality reduction methods, particularly principal component analysis (PCA), as obfuscation functions.
Linear dimensionality reduction is theoretically tractable and simplifies the construction of integration functions.
However, it has a structural vulnerability: although the analyst normally receives only the intermediate representations of the anchor dataset, if the original anchor dataset---which is commonly held across all parties---is also accessed by the analyst, the linear obfuscation function can be estimated by applying the pseudoinverse of the anchor data to its intermediate representations.
In particular, if a party colludes with the analyst such that the number of leaked original anchor instances substantially exceeds the feature dimensionality, the linear obfuscation function can be completely identified and the original datasets of other parties reconstructed~\cite{DC-odc, common_obfuscation}.

Nonlinear dimensionality reduction is a promising way to mitigate reconstruction risk because a single global linear inverse mapping is generally insufficient for recovering original datasets.
However, simply replacing linear dimensionality reduction with nonlinear dimensionality reduction is not enough.
Previous attempts to introduce nonlinear dimensionality reduction into DC analysis reported degradation in analytical accuracy compared with linear methods~\cite{Mashiko2025}.
A possible reason is that nonlinear dimensionality reduction changes local coordinate structures in a party-specific manner, and a single linear integration function cannot generally align the resulting nonlinear intermediate representations.
Therefore, a nonlinear integration method is needed to construct collaboration representations from nonlinear intermediate representations.

In addition to this limitation, existing integration methods are mainly designed to minimize discrepancies among the collaboration representations of the anchor dataset across parties.
However, they do not provide a mechanism for incorporating downstream-relevant information, such as target variables or external sample-wise relationships, into the target representation.

Taken together, the above studies substantiate the three limitations outlined in the Introduction: unresolved issues in linear integration, difficulty in handling nonlinear intermediate representations, and the lack of a mechanism for incorporating downstream-relevant information into the target representation.
These limitations motivate the proposed framework, which addresses linear integration, nonlinear integration, and downstream-relevant target representation design within a unified optimization-based formulation.

\section{Data Collaboration (DC) Analysis}
\label{sec:dca}
This section describes the DC analysis framework considered in this study and introduces the notation used to formulate the proposed integration methods.
Fig.~\ref{fig:workflow} provides an overview of the DC analysis framework considered in this study.

\begin{figure*}[t]
\centering

\definecolor{partycolor}{HTML}{7030A0}
\definecolor{datacolor}{HTML}{222222}
\definecolor{analystcolor}{HTML}{2CA02C}

\resizebox{0.82\textwidth}{!}{%
\begin{tikzpicture}[
  font=\sffamily,
  >={Stealth[length=2.4mm]},
  box/.style={draw=#1, fill=white, rounded corners=2pt,
              line width=0.9pt, minimum width=1.0cm,
              minimum height=0.82cm, align=center,
              inner sep=1.5pt, font=\small\bfseries},
  splitbox/.style={draw=#1, fill=white, rounded corners=2pt,
              line width=0.9pt, minimum width=1.0cm,
              minimum height=0.82cm, align=center,
              inner sep=1.5pt, font=\small\bfseries,
              path picture={
                \draw[datacolor!95, dash pattern=on 0.7pt off 0.7pt, line width=0.6pt]
                  (path picture bounding box.west)
                  -- (path picture bounding box.east);
              }},
  proc/.style={circle, draw=#1, fill=#1!85, text=white,
               line width=0.8pt, minimum size=0.52cm,
               inner sep=0pt, font=\small\itshape},
  arr/.style={-{Stealth[length=2.4mm]}, line width=0.9pt},
  lineonly/.style={line width=0.9pt},
  bidir/.style={{Stealth[length=2.0mm]}-{Stealth[length=2.0mm]},
                line width=0.9pt, draw=datacolor}
]

\newcommand{\hospital}[2]{%
  \begin{scope}[shift={(#1,#2)}, scale=0.38]
    \fill[partycolor] (-0.55,-0.70) rectangle (0.55,0.70);
    \fill[white] (-0.13,-0.70) rectangle (0.13,-0.35);
    \fill[white] (-0.34,0.10) rectangle (0.34,0.28);
    \fill[white] (-0.10,-0.14) rectangle (0.10,0.52);
  \end{scope}
}

\node[font=\bfseries, text=partycolor] at (2.35,5.65) {Each Party};
\node[font=\bfseries, text=analystcolor] at (8.85,5.65) {Analyst};

\draw[partycolor!45, dashed, rounded corners=8pt, line width=0.8pt]
  (0.53,0.55) rectangle (4.40,5.25);

\draw[analystcolor, fill=analystcolor!10, rounded corners=8pt, line width=1.0pt]
  (4.85,0.65) rectangle (12.95,5.25);

\foreach \i/\y/\lab in {1/4.55/1,2/3.25/2,K/1.25/K} {
  \hospital{0.15}{\y}

  \draw[partycolor, fill=partycolor!11, rounded corners=5pt, line width=0.9pt]
    (0.68,\y-0.58) rectangle (4.23,\y+0.58);

  \node[splitbox=partycolor] (raw\i) at (1.35,\y)
    {$\symbfit{X}^{(\lab)}$\\[0.5mm]$\symbfit{A}$};

  \node[proc=datacolor] (f\i) at (2.25,\y) {$f_{\lab}$};

  \node[splitbox=datacolor, fill=datacolor!8] (local\i) at (3.35,\y)
    {$\tilde{\symbfit{X}}^{(\lab)}$\\[0.5mm]
     $\tilde{\symbfit{A}}^{(\lab)}$};

  \draw[lineonly, draw=datacolor] (raw\i) -- (f\i);
  \draw[arr, draw=datacolor] (f\i) -- (local\i);
  \draw[arr, draw=datacolor!70, line width=1.2pt] (local\i.east) -- (5.05,\y);
}

\node[font=\Large] at (1.35,2.25) {$\vdots$};
\node[font=\Large] at (3.35,2.25) {$\vdots$};

\foreach \i/\y/\lab in {1/4.55/1,2/3.25/2,K/1.25/K} {
  \node[splitbox=datacolor, fill=datacolor!8] (recv\i) at (5.55,\y)
    {$\tilde{\symbfit{X}}^{(\lab)}$\\[0.5mm]
     $\tilde{\symbfit{A}}^{(\lab)}$};

  \node[proc=analystcolor] (g\i) at (6.45,\y) {$g_{\lab}$};

  \node[splitbox=analystcolor, fill=white] (hat\i) at (7.55,\y)
    {$\hat{\symbfit{X}}^{(\lab)}$\\[0.5mm]
     $\hat{\symbfit{A}}^{(\lab)}$};

  \draw[lineonly, draw=analystcolor] (recv\i) -- (g\i);
  \draw[arr, draw=analystcolor] (g\i) -- (hat\i);
}

\node[font=\Large] at (5.55,2.25) {$\vdots$};
\node[font=\Large] at (7.55,2.25) {$\vdots$};

\node[box=analystcolor, minimum width=1.15cm, minimum height=2.15cm] (stack) at (9.45,3.85)
  {$\hat{\symbfit{X}}^{(1)}$\\[1.0mm]
   $\hat{\symbfit{X}}^{(2)}$\\[1.0mm]
   $\vdots$\\[1.0mm]
   $\hat{\symbfit{X}}^{(K)}$};

\node[box=analystcolor, minimum width=1.05cm, minimum height=0.62cm] (z) at (9.45,1.25)
  {$\symbfit{Z}$};

\node[proc=analystcolor, minimum size=0.62cm] (h) at (10.70,3.85) {$h$};

\node[box=analystcolor, minimum width=1.1cm, minimum height=1.9cm] (yout) at (11.95,3.85)
  {$\symbfit{y}^{(1)}$\\[1.0mm]
   $\symbfit{y}^{(2)}$\\[1.0mm]
   $\vdots$\\[1.0mm]
   $\symbfit{y}^{(K)}$};

\coordinate (hat1x) at ($(hat1.east)+(0,0.21)$);
\coordinate (hat2x) at ($(hat2.east)+(0,0.21)$);
\coordinate (hatKx) at ($(hatK.east)+(0,0.21)$);

\coordinate (hat1a) at ($(hat1.east)+(0,-0.21)$);
\coordinate (hat2a) at ($(hat2.east)+(0,-0.21)$);
\coordinate (hatKa) at ($(hatK.east)+(0,-0.21)$);

\draw[bidir] (hat1a) -- (z.north west);
\draw[bidir] (hat2a) -- ($(z.north west)!0.55!(z.west)$);
\draw[bidir] (hatKa) -- (z.west);

\draw[arr, draw=analystcolor] (hat1x) -- (stack.west |- hat1x);
\draw[arr, draw=analystcolor] (hat2x) -- (stack.west);
\draw[arr, draw=analystcolor] (hatKx) -- (stack.south west);

\draw[lineonly, draw=analystcolor] (stack.east) -- (h.west);
\draw[arr, draw=analystcolor] (h.east) -- (yout.west);

\draw[draw=analystcolor, line width=1.4pt, -{Stealth[length=3mm]}]
  (10.20,0.65)
  .. controls (8.0,-0.25) and (4.5,-0.35) ..
  (2.35,0.48);

\node[draw=analystcolor, fill=analystcolor!85, text=white,
      rounded corners=2pt, line width=0.8pt,
      inner xsep=4pt, inner ysep=2pt,
      font=\small\itshape] (feedback) at (8.90,0.18)
  {$g_1,\;g_2,\dots,g_K, \;h$};

\end{tikzpicture}%
}
\caption{Workflow of DC analysis.}
\label{fig:workflow}
\end{figure*}

Throughout this paper, we denote by $[n]$ the set of positive integers $\{1, 2, \ldots, n\}$.
Here and hereafter, $\symbfit{I}_{m}$ denotes the $m$-dimensional identity matrix.
For a matrix $\symbfit X$, $\symbfit X_{i,:}$ and $\symbfit X_{:,j}$ denote its $i$-th row vector and $j$-th column vector, respectively, and $\symbfit X^\dagger$ denotes its Moore--Penrose pseudoinverse.
We consider $K$ parties, where each party $k \in [K]$ holds a $d$-dimensional feature matrix $\symbfit{X}^{(k)}\in\mathbb{R}^{n(k)\times d}$ and a target variable vector $\symbfit{y}^{(k)}\in\mathbb{R}^{n(k)}$ for $n(k)$ instances.

In DC analysis, an anchor dataset $\symbfit{A}\in\mathbb{R}^{n_{\mathrm{a}}\times d}$ is introduced as pseudo-data that provides a common reference for integrating the party-specific representations.
It has the same feature dimension as the original datasets but an arbitrary number of rows, and can be constructed, for example, by random sampling or oversampling of publicly shareable data.
The anchor dataset is shared among parties but is not disclosed to the analyst.

To protect data confidentiality, each party $k \in [K]$ constructs its own obfuscation function $f_{k}:\mathbb{R}^{1\times d}\to\mathbb{R}^{1\times \tilde d(k)}$.
The function is applied to matrices row-wise as follows:
\begin{equation}
\label{eq:row-wise-obfuscation}
f_k(\symbfit{X})
\coloneqq
\begin{pmatrix}
f_k(\symbfit{X}_{1,:}) \\
f_k(\symbfit{X}_{2,:}) \\
\vdots \\
f_k(\symbfit{X}_{n,:})
\end{pmatrix}
\in \mathbb{R}^{n \times \tilde d(k)} ,
\end{equation}
where $\symbfit{X}\in\mathbb{R}^{n\times d}$.

Applying $f_k$ to both the original dataset and the anchor dataset yields the following intermediate representations:
\begin{equation}
\label{eq:intermediate-representations}
\begin{aligned}
\tilde{\symbfit{X}}^{(k)} &\coloneqq f_{k}(\symbfit{X}^{(k)})
\in\mathbb{R}^{n(k)\times \tilde d(k)}, \\
\tilde{\symbfit{A}}^{(k)} &\coloneqq f_{k}(\symbfit{A})
\in\mathbb{R}^{n_{\mathrm{a}}\times \tilde d(k)} .
\end{aligned}
\end{equation}
Each party $k \in [K]$ keeps its original dataset, anchor dataset, and obfuscation function confidential, and provides only the intermediate representations together with the target variable vector $\symbfit y^{(k)}$ to the analyst.
This restriction is important because simultaneous access to the anchor dataset and its intermediate representations could allow the analyst to estimate the obfuscation function.

The analyst constructs, for each party $k \in [K]$, an integration function $g_{k}:\mathbb{R}^{1\times \tilde d(k)}\to\mathbb{R}^{1\times \hat d}$ that maps its representation into a common feature space.
The corresponding collaboration representations are defined as follows:
\begin{equation}
\label{eq:collaboration-representations}
\begin{aligned}
\hat{\symbfit{X}}^{(k)} &\coloneqq g_{k}(\tilde{\symbfit{X}}^{(k)})
\in\mathbb{R}^{n(k)\times \hat d}, \\
\hat{\symbfit{A}}^{(k)} &\coloneqq g_{k}(\tilde{\symbfit{A}}^{(k)})
\in\mathbb{R}^{n_{\mathrm{a}}\times \hat d} .
\end{aligned}
\end{equation}
The integration function $g_k$ is not chosen arbitrarily; it must align the intermediate representations, which differ across parties, into a common feature space.
Ideally, the compositions of the obfuscation and integration functions are expected to approximately coincide across all parties:
\begin{equation}
\label{eq:ideal-concordance}
g_k \circ f_{k}
\approx
g_{k'} \circ f_{k'}
\qquad
(k,k'\in[K]).
\end{equation}
When Eq.~\eqref{eq:ideal-concordance} holds, essentially the same feature transformation is applied to the original datasets of all parties, and the collaboration representations can be treated as coherent observations in a common feature space.

However, the analyst cannot directly observe the obfuscation function $f_k$ itself; the analyst observes only its output on the anchor dataset, $\tilde{\symbfit{A}}^{(k)} = f_{k}(\symbfit{A})$, and therefore Eq.~\eqref{eq:ideal-concordance} cannot be directly enforced.
Instead, we adopt concordance on the anchor dataset as a surrogate objective and design the integration functions to satisfy:
\begin{equation}
\label{eq:anchor-concordance}
g_{k}(\tilde{\symbfit{A}}^{(k)})
\approx
g_{k'}(\tilde{\symbfit{A}}^{(k')})
\qquad
(k,k'\in[K]).
\end{equation}

In practice, rather than handling Eq.~\eqref{eq:anchor-concordance} for each pair of parties, we introduce a target representation $\symbfit{Z}\in\mathbb{R}^{n_{\mathrm{a}}\times \hat d}$ of anchor data in the common feature space.
Each integration function is then trained so that the collaboration representations of the anchor dataset approach this common target:
\begin{equation}
\label{eq:target-representation-alignment}
g_k(\tilde{\symbfit{A}}^{(k)}) \approx \symbfit{Z}
\qquad(k\in[K]).
\end{equation}

Once the integration function $g_k$ for each party $k\in[K]$ is obtained, the analyst applies it to the intermediate representations $\tilde{\symbfit{X}}^{(k)}$ to obtain the collaboration representations $\hat{\symbfit{X}}^{(k)}$.
Since these collaboration representations lie in approximately the same common feature space, the analyst vertically concatenates them to construct the following single integrated dataset:
\begin{equation}
\label{eq:single-integrated-dataset}
\hat{\symbfit{X}}
=
\begin{pmatrix}
\hat{\symbfit{X}}^{(1)}\\
\hat{\symbfit{X}}^{(2)}\\
\vdots\\
\hat{\symbfit{X}}^{(K)}
\end{pmatrix},
\qquad
\hat{\symbfit{y}}
=
\begin{pmatrix}
\symbfit{y}^{(1)}\\
\symbfit{y}^{(2)}\\
\vdots\\
\symbfit{y}^{(K)}
\end{pmatrix} .
\end{equation}
The analyst trains a machine learning model $h:\mathbb{R}^{1\times \hat d}\to\mathbb{R}$ using the single integrated dataset in Eq.~\eqref{eq:single-integrated-dataset} and then returns the integration function $g_k$ and the trained model $h$ to each party $k \in [K]$, as indicated in Fig.~\ref{fig:workflow}.
By composing these with its own obfuscation function $f_k$, each party obtains a model $h \circ g_k \circ f_k:\mathbb{R}^{1\times d}\to\mathbb{R}$ applicable directly to its original dataset.

Through this procedure, each party $k\in[K]$ makes predictions on new data independently, without any additional communication with the analyst.

\subsection{Existing Linear Integration Formulations}
\label{subsec:existing-linear-integration}

Existing integration methods for DC analysis are based on linear integration functions.
We next summarize these formulations.
Motivated by Eq.~\eqref{eq:target-representation-alignment}, suppose that the integration function is a linear map.
For each party $k\in[K]$, define
\begin{equation}
\label{eq:lki-linear-map}
g_k(\symbfit{x}) = \symbfit{x}\symbfit{G}^{(k)},
\qquad
\symbfit{G}^{(k)}\in\mathbb{R}^{\tilde d(k)\times \hat d},
\end{equation}
where $\symbfit{x}\in\mathbb{R}^{1\times\tilde d(k)}$ and $\symbfit{G}^{(k)}$ is a coefficient matrix.
Then, the collaboration representations of the anchor dataset are given by
\begin{equation}
\label{eq:lki-anchor-representation}
g_k(\tilde{\symbfit{A}}^{(k)})
=
\tilde{\symbfit{A}}^{(k)}\symbfit{G}^{(k)} .
\end{equation}
A natural least-squares formulation is to estimate both the coefficient matrices and a common target representation $\symbfit{Z}\in\mathbb{R}^{n_{\mathrm{a}}\times\hat d}$ as follows:
\begin{equation}
\label{eq:linear-naive-formulation}
\min_{\symbfit{Z},\,\{\symbfit{G}^{(k)}\}_{k=1}^{K}}
\sum_{k=1}^{K}
\left\|
\tilde{\symbfit{A}}^{(k)}\symbfit{G}^{(k)}
-
\symbfit{Z}
\right\|_{\mathrm{F}}^2 .
\end{equation}
Without additional constraints, however, Eq.~\eqref{eq:linear-naive-formulation} admits the trivial optimal solution $\symbfit{Z}=\symbfit{0}$ and $\symbfit{G}^{(k)}=\symbfit{0}$.
Existing linear integration methods avoid this degeneracy by prescribing the target representation or by imposing constraints on the coefficient matrices.

In MPP~\cite{DCframework2}, let $\symbfit W_A=[\tilde{\symbfit A}^{(1)},\ldots,\tilde{\symbfit A}^{(K)}]$.
The target representation is chosen from the leading left singular subspace of $\symbfit W_A$; specifically, if $\symbfit U_{\hat d}$ is an orthonormal basis of this subspace, then $\symbfit Z=\symbfit U_{\hat d}\symbfit R$ with an arbitrary invertible matrix $\symbfit R\in\mathbb R^{\hat d\times\hat d}$ spans the same subspace.
For this prescribed $\symbfit Z$, each unconstrained least-squares subproblem has the minimum-norm solution
\begin{equation}
\label{eq:fixed-Z-linear-solution}
\symbfit G^{(k)}
=
(\tilde{\symbfit A}^{(k)})^\dagger \symbfit Z
\qquad (k\in[K]).
\end{equation}
Thus, MPP retains a right-side invertible degree of freedom in the target representation.
Because $\symbfit R$ is arbitrary, different choices of the basis and scaling within the same subspace can substantially affect downstream accuracy~\cite{DC-odc}.

ODC~\cite{DC-odc} reduces this invertible degree of freedom by imposing column-orthogonality constraints on the coefficient matrices, typically $\symbfit G^{(k)\top}\symbfit G^{(k)}=\symbfit I_{\hat d}$ for $k\in[K]$.
However, ODC still treats $\symbfit Z$ as a prescribed target representation.
For example, when the dimensions match, the intermediate representation of a reference party, such as $\tilde{\symbfit A}^{(1)}$, can be used as $\symbfit Z$.
Thus, the result depends on how this target is chosen.
Moreover, ODC relies on the strong assumption that the dimensionality reduction performed by each party is compatible with an orthogonal transformation; when this assumption is violated, analytical accuracy may substantially decrease.

In contrast, the GEP-based method~\cite{KawakamiDC} does not prescribe a target representation.
Instead, it directly minimizes discrepancies among the anchor collaboration representations.
In the notation of this paper, its column-wise formulation can be summarized as
\begin{equation}
\label{eq:existing-gep-formulation}
\begin{aligned}
\min \quad
&
\sum_{k=1}^{K}
\sum_{k'=1}^{K}
\left\|
\tilde{\symbfit A}^{(k)}\symbfit G^{(k)}
-
\tilde{\symbfit A}^{(k')}\symbfit G^{(k')}
\right\|_{\mathrm F}^{2}
\\[2mm]
\text{s.~t.}\quad
&
\sum_{k=1}^{K}
\left\|
\tilde{\symbfit A}^{(k)}
\symbfit G^{(k)}_{:,j}
\right\|_2^2
=
1
\quad(j\in[\hat d]).
\end{aligned}
\end{equation}
where $\symbfit G^{(k)}_{:,j}$ denotes the $j$-th column of $\symbfit G^{(k)}$.
This formulation is reduced to a generalized eigenvalue problem.
However, the constraints in Eq.~\eqref{eq:existing-gep-formulation} normalize each integration component separately and do not explicitly impose orthogonality or linear independence among different components, although the generalized eigenvalue problem may yield orthogonal components implicitly.
Moreover, the formulation relies on rank conditions on the intermediate anchor representations, such as full column rank of $\tilde{\symbfit A}^{(k)}$, and an explicit global optimality guarantee for the solution obtained from the generalized eigenvalue problem has not been established.

These observations motivate a linear integration formulation that does not prescribe $\symbfit Z$ in advance.
Such a formulation should avoid unnecessary invertible degrees of freedom and strong orthogonality or rank assumptions, while still admitting an analytically derived globally optimal solution.

\section{Proposed Method}
\label{sec:proposed-method}

This section presents the proposed integration methods for DC analysis.
We first formulate linear target-normalized integration (LTI) as a linear integration problem and give its globally optimal solution.
We then extend LTI to kernel-based target-normalized integration (KTI) by kernelizing the integration function.
Finally, we introduce graph-based regularization and a centering constraint to incorporate geometric and target-variable information into the target representation.

\subsection{Linear Target-Normalized Integration}

Motivated by the limitations of existing linear integration formulations
discussed in Section~\ref{subsec:existing-linear-integration}, we return to
the least-squares formulation in Eq.~\eqref{eq:linear-naive-formulation} and
impose the orthonormality constraint
\begin{equation}
\label{eq:lti-orthonormality}
\symbfit Z^{\top}\symbfit Z=\symbfit I_{\hat d}.
\end{equation}
As a result, we formulate the following linear integration problem:
\begin{equation}
\label{eq:linear-formulation}
\begin{aligned}
\min \quad
&
\sum_{k=1}^{K}
\left\|
\tilde{\symbfit A}^{(k)}\symbfit G^{(k)}
-
\symbfit Z
\right\|_{\mathrm F}^2
\\[2mm]
\text{s.~t.} \quad
&
\symbfit G^{(k)}
\in
\mathbb R^{\tilde d(k)\times\hat d}
\quad(k\in[K]),
\\
&
\symbfit Z^{\top}\symbfit Z=\symbfit I_{\hat d},
\\
&
\symbfit Z\in\mathbb R^{n_{\mathrm a}\times\hat d}.
\end{aligned}
\end{equation}

The objective in Eq.~\eqref{eq:linear-formulation} measures the discrepancy
between the collaboration representations of the anchor dataset and the target
representation.
The constraint in Eq.~\eqref{eq:lti-orthonormality} prevents the zero solution
and avoids selecting the same direction multiple times.

The problem in Eq.~\eqref{eq:linear-formulation} admits a simple spectral
solution. For a fixed target matrix $\symbfit Z$, one least-squares minimizer
with respect to $\symbfit G^{(k)}$ is the minimum-norm solution
\begin{equation}
\label{eq:lti-G-fixed-Z}
\symbfit G^{(k)\star}
=
(\tilde{\symbfit A}^{(k)})^\dagger
\symbfit Z.
\end{equation}
Let
\begin{equation}
\label{eq:lti-projector}
\symbfit P^{(k)}
\coloneqq
\tilde{\symbfit A}^{(k)}
(\tilde{\symbfit A}^{(k)})^\dagger
\end{equation}
be the orthogonal projector onto
$\mathrm{Col}(\tilde{\symbfit A}^{(k)})$.
Although the least-squares minimizer may not be unique, all minimizers yield
the same fitted anchor representation $\symbfit P^{(k)}\symbfit Z$ and hence
the same residual.
Thus, substituting the minimum-norm representative in
Eq.~\eqref{eq:lti-G-fixed-Z} gives
\begin{equation}
\label{eq:lti-residual-projection}
\tilde{\symbfit A}^{(k)}
\symbfit G^{(k)\star}
-
\symbfit Z
=
-
(\symbfit I_{n_{\mathrm a}}-\symbfit P^{(k)})
\symbfit Z.
\end{equation}
Since $\symbfit P^{(k)}$ is an orthogonal projector, it satisfies
$\symbfit P^{(k)\top}=\symbfit P^{(k)}$ and
$(\symbfit P^{(k)})^2=\symbfit P^{(k)}$.
Therefore,
\begin{equation}
\label{eq:lti-projection-norm-trace}
\left\|
(\symbfit I_{n_{\mathrm a}}-\symbfit P^{(k)})
\symbfit Z
\right\|_{\mathrm F}^{2}
=
\mathrm{tr}
\left(
\symbfit Z^{\top}
(\symbfit I_{n_{\mathrm a}}-\symbfit P^{(k)})
\symbfit Z
\right).
\end{equation}
Summing Eq.~\eqref{eq:lti-projection-norm-trace} over $k=1,\ldots,K$ and
using $\symbfit Z^{\top}\symbfit Z=\symbfit I_{\hat d}$, we obtain
\begin{equation}
\label{eq:lti-reduced-objective}
\sum_{k=1}^{K}
\left\|
(\symbfit I_{n_{\mathrm a}}-\symbfit P^{(k)})
\symbfit Z
\right\|_{\mathrm F}^{2}
=
K\hat d
-
\mathrm{tr}
\left(
\symbfit Z^{\top}
\symbfit M
\symbfit Z
\right),
\end{equation}
where
\begin{equation}
\label{eq:lti-M-definition}
\symbfit M
\coloneqq
\sum_{k=1}^{K}
\symbfit P^{(k)}.
\end{equation}
Since $K\hat d$ is independent of $\symbfit Z$, minimizing
Eq.~\eqref{eq:lti-reduced-objective} is equivalent to maximizing the second
term. Hence, Eq.~\eqref{eq:linear-formulation} is reduced to
\begin{equation}
\label{eq:lti-trace-max}
\begin{aligned}
\max \quad
&
\mathrm{tr}
\left(
\symbfit Z^{\top}
\symbfit M
\symbfit Z
\right)
\\[2mm]
\text{s.~t.} \quad
&
\symbfit Z^{\top}\symbfit Z
=
\symbfit I_{\hat d},
\\
&
\symbfit Z\in\mathbb R^{n_{\mathrm a}\times\hat d}.
\end{aligned}
\end{equation}

Since each $\symbfit P^{(k)}$ is an orthogonal projector,
$\symbfit M$ is symmetric positive semidefinite.
To see the spectral form of the solution, first consider the single-column
case:
\begin{equation}
\label{eq:lti-single-column}
\begin{aligned}
\max \quad
&
\symbfit z^{\top}
\symbfit M
\symbfit z
\\[1mm]
\text{s.~t.} \quad
&
\|\symbfit z\|_2^2=1.
\end{aligned}
\end{equation}
The Lagrangian is
\begin{equation}
L(\symbfit z,\eta)
=
\symbfit z^{\top}
\symbfit M
\symbfit z
-
\eta
(
\symbfit z^{\top}
\symbfit z
-
1
).
\end{equation}
The stationary condition is
\begin{equation}
\symbfit M\symbfit z
=
\eta
\symbfit z.
\end{equation}
Moreover, for a unit-norm eigenvector,
$\symbfit z^{\top}\symbfit M\symbfit z=\eta$.
Therefore, the single-column optimum is attained by an eigenvector associated
with the largest eigenvalue of $\symbfit M$.

For $\hat d>1$, the constraint
$\symbfit Z^{\top}\symbfit Z=\symbfit I_{\hat d}$ requires the columns of
$\symbfit Z$ to be mutually orthonormal.
For a symmetric matrix $\symbfit M$, the trace maximization in
Eq.~\eqref{eq:lti-trace-max} is solved by choosing $\hat d$ orthonormal
eigenvectors associated with its $\hat d$ largest eigenvalues.
Thus, an optimal $\symbfit Z$ is obtained by taking these eigenvectors as its
columns.

To compute these eigenvectors efficiently, let
$\symbfit Q^{(k)}\in\mathbb R^{n_{\mathrm a}\times r(k)}$ be a matrix whose
columns form an orthonormal basis of
$\mathrm{Col}(\tilde{\symbfit A}^{(k)})$, where
$r(k)=\operatorname{rank}(\tilde{\symbfit A}^{(k)})$.
Then
\begin{equation}
\label{eq:lti-P-Q}
\symbfit P^{(k)}
=
\symbfit Q^{(k)}
\symbfit Q^{(k)\top}.
\end{equation}
Define
\begin{equation}
\label{eq:lti-WQ-definition}
\symbfit W_Q
\coloneqq
[
\symbfit Q^{(1)},\ldots,\symbfit Q^{(K)}
].
\end{equation}
Then
\begin{equation}
\label{eq:lti-M-WQ}
\symbfit M
=
\sum_{k=1}^{K}
\symbfit Q^{(k)}
\symbfit Q^{(k)\top}
=
\symbfit W_Q
\symbfit W_Q^{\top}.
\end{equation}
Therefore, the dominant eigenvectors of $\symbfit M$ can be obtained as the
dominant left singular vectors of $\symbfit W_Q$.

Let $\symbfit U_{Q\hat d}$ denote a matrix whose columns are such dominant
left singular vectors.
For any $\hat d\times\hat d$ orthogonal matrix $\symbfit O$,
$\symbfit U_{Q\hat d}\symbfit O$ has the same objective value and still
satisfies the orthonormality constraint.
Thus, the target representation is determined only up to an orthogonal
rotation.

\begin{theorem}[Global Optimal Solution of the Linear Integration Problem]
\label{thm:linear-integration-solution}
Assume that $\hat d\le n_{\mathrm a}$.
For each $k\in[K]$, let
$\symbfit Q^{(k)}\in\mathbb R^{n_{\mathrm a}\times r(k)}$ be a matrix whose
columns form an orthonormal basis of
$\mathrm{Col}(\tilde{\symbfit A}^{(k)})$, where
$r(k)=\operatorname{rank}(\tilde{\symbfit A}^{(k)})$.
Define
\begin{equation}
\label{eq:WQ-definition}
\symbfit W_Q
\coloneqq
[
\symbfit Q^{(1)},\ldots,\symbfit Q^{(K)}
].
\end{equation}
Let
$\symbfit U_{Q\hat d}\in\mathbb R^{n_{\mathrm a}\times\hat d}$
be one choice of the $\hat d$ left singular vectors of $\symbfit W_Q$
corresponding to its $\hat d$ largest singular values.
Let
\[
\mathbf O(\hat d)
\coloneqq
\{
\symbfit O\in\mathbb R^{\hat d\times\hat d}
\mid
\symbfit O^\top\symbfit O=\symbfit I_{\hat d}
\}.
\]
Then, for any $\symbfit O\in\mathbf O(\hat d)$, a globally optimal solution to
Eq.~\eqref{eq:linear-formulation} is given by
\begin{equation}
\label{eq:lti-global-solution}
\symbfit Z^\star
=
\symbfit U_{Q\hat d}
\symbfit O,
\qquad
\symbfit G^{(k)\star}
=
(\tilde{\symbfit A}^{(k)})^\dagger
\symbfit Z^\star
\quad
(k\in[K]).
\end{equation}
For each fixed $\symbfit Z^\star$, the coefficient matrix
$\symbfit G^{(k)\star}$ is the minimum-norm representative among the
least-squares minimizers.
The arbitrary orthogonal matrix $\symbfit O$ reflects the rotational
indeterminacy of the target representation.
When the $\hat d$-th largest singular value of $\symbfit W_Q$ is separated from
the next one, this form characterizes all globally optimal
$\symbfit Z^\star$ up to an orthogonal rotation.
\end{theorem}

Theorem~\ref{thm:linear-integration-solution} follows from the above reduction,
the identity $\symbfit M=\symbfit W_Q\symbfit W_Q^\top$, and the Ky Fan maximum
principle.
Thus, the proposed method constructs the target representation from the
dominant left singular subspace of $\symbfit W_Q$, or equivalently from the
dominant eigenspace of the sum-of-projectors matrix $\symbfit M$.
We refer to this linear integration method as Linear Target-Normalized
Integration (LTI).
Algorithm~\ref{alg:lti} summarizes the resulting procedure for constructing LTI functions.

\begin{algorithm}[t]
\caption{Construction of LTI functions}
\label{alg:lti}
\begin{algorithmic}[1]
\Require Intermediate representations of the anchor dataset $\{\tilde{\symbfit A}^{(k)}\}_{k=1}^K$
\Ensure Target representation $\symbfit Z^\star$ and integration functions $\{g_k^\star\}_{k=1}^K$

\For{$k=1,\ldots,K$}
    \State Compute an orthonormal basis $\symbfit Q^{(k)}\in\mathbb R^{n_{\mathrm{a}}\times r(k)}$ of $\mathrm{Col}(\tilde{\symbfit A}^{(k)})$.
\EndFor
\State Construct $\symbfit W_Q\gets[\symbfit Q^{(1)},\ldots,\symbfit Q^{(K)}]$.
\State Compute $\symbfit Z^\star\gets\symbfit U_{Q\hat{d}}$, the matrix of the $\hat d$ left singular vectors of $\symbfit W_Q$ corresponding to its $\hat d$ largest singular values.
\For{$k=1,\ldots,K$}
    \State Compute $\symbfit G^{(k)\star}\gets(\tilde{\symbfit A}^{(k)})^\dagger\symbfit Z^\star$.
    \State Define $g_k^\star(\symbfit x)\gets\symbfit x\symbfit G^{(k)\star}$.
\EndFor
\end{algorithmic}
\end{algorithm}

\subsection{Kernel-based Target-Normalized Integration}

The linear formulation in Eq.~\eqref{eq:linear-formulation} clarifies the optimization structure of the target representation.
From a data confidentiality perspective, however, it is desirable to integrate intermediate representations obtained via nonlinear dimensionality reduction.
We therefore extend LTI to a nonlinear integration method by kernelizing the integration function using reproducing kernel Hilbert spaces (RKHSs).

For each party $k\in[K]$, let $\mathcal{H}_k$ denote the RKHS associated with the reproducing kernel
\begin{equation}
\label{eq:kernel-definition}
\kappa_k:
\mathbb{R}^{1\times\tilde d(k)}
\times
\mathbb{R}^{1\times\tilde d(k)}
\to
\mathbb{R} .
\end{equation}
That is, $\mathcal{H}_k$ is a Hilbert space of functions on $\mathbb{R}^{1\times\tilde d(k)}$ in which $\kappa_k$ acts as the reproducing kernel: for every $f\in\mathcal{H}_k$ and $\symbfit{x}\in\mathbb{R}^{1\times\tilde d(k)}$, $f(\symbfit{x})=\langle f,\kappa_k(\cdot,\symbfit{x})\rangle_{\mathcal{H}_k}$.
The linear formulation in Eq.~\eqref{eq:linear-formulation} represents each integration function by the linear term $\tilde{\symbfit A}^{(k)}\symbfit G^{(k)}$.
To handle intermediate representations obtained through nonlinear dimensionality reduction, we replace this linear term with a nonlinear function $g_k(\tilde{\symbfit A}^{(k)})$ defined in an RKHS and add an RKHS norm penalty to control the complexity of $g_k$.
Because the collaboration representation is $\hat d$-dimensional, we let $g_k$ take values in the product space $\mathcal{H}_k^{\hat d}$, whose norm is given by the sum of the squared $\mathcal{H}_k$-norms of its $\hat d$ coordinate functions.
The nonlinear integration problem is then formulated as follows:
\begin{equation}
\label{eq:kernel-formulation}
\begin{aligned}
\min \quad
&\sum_{k=1}^{K}
\left(
\left\|
g_k(\tilde{\symbfit{A}}^{(k)})
-
\symbfit{Z}
\right\|_{\mathrm{F}}^2
+
\lambda
\|g_k\|_{\mathcal{H}_k^{\hat d}}^2
\right)
\\[2mm]
\text{s.~t.} \quad
&g_k\in\mathcal{H}_k^{\hat d}\quad(k\in[K]), \\
&\symbfit{Z}^{\top}\symbfit{Z}=\symbfit{I}_{\hat d}, \\
&\symbfit{Z}\in\mathbb{R}^{n_{\mathrm{a}}\times\hat d},
\end{aligned}
\end{equation}
where $\lambda>0$ is a regularization parameter controlling the RKHS norm penalty.
The problem in Eq.~\eqref{eq:kernel-formulation} is solved by an analogous two-step reduction.
To state this reduction, define the kernel matrix for each $k\in[K]$ as
\begin{equation}
\label{eq:kernel-matrix-definition}
\symbfit \Sigma_k
=
(\kappa_k(\tilde{\symbfit A}^{(k)}_{i,:},\tilde{\symbfit A}^{(k)}_{i',:}))_{(i,i') \in [n_{\mathrm{a}}] \times [n_{\mathrm{a}}]}
\in\mathbb R^{n_{\mathrm a}\times n_{\mathrm a}},
\end{equation}
For any $\symbfit x\in\mathbb R^{1\times\tilde d(k)}$, define
\begin{equation}
\label{eq:kernel-vector-definition}
\symbfit\kappa_k(\symbfit x)
=
\bigl(
\kappa_k(\symbfit x,\tilde{\symbfit A}^{(k)}_{1,:}),
\ldots,
\kappa_k(\symbfit x,\tilde{\symbfit A}^{(k)}_{n_{\mathrm{a}},:})
\bigr)
\in\mathbb R^{1\times n_{\mathrm{a}}}.
\end{equation}
Finally, let
\begin{equation}
\label{eq:Sk-Mlambda-definition}
\symbfit S_k \coloneqq (\symbfit \Sigma_k+\lambda\symbfit I_{n_{\mathrm{a}}})^{-1},
\qquad
\symbfit M_\lambda \coloneqq \lambda\sum_{k=1}^K \symbfit S_k .
\end{equation}

For fixed $\symbfit Z$, each integration function is obtained by solving
\begin{equation}
\label{eq:kernel-fixed-Z-problem}
\min_{g_k\in\mathcal{H}_k^{\hat{d}}}
\left\|
g_k(\tilde{\symbfit A}^{(k)})
-
\symbfit Z
\right\|_{\mathrm F}^2
+
\lambda
\|g_k\|_{\mathcal{H}_k^{\hat{d}}}^2 .
\end{equation}
Because $\lambda>0$, this problem has a unique minimizer in
$\mathcal{H}_k^{\hat d}$.
In the representer form, we use the coefficient matrix
\[
\symbfit\Gamma_k^\star=\symbfit S_k\symbfit Z,
\]
which yields
\begin{equation}
\label{eq:gk-vector-solution}
g_k^\star(\symbfit x)
=
\symbfit\kappa_k(\symbfit x)\symbfit S_k\symbfit Z .
\end{equation}
Second, substituting these optimal integration functions into Eq.~\eqref{eq:kernel-formulation} reduces the problem over $\symbfit Z$ to the constrained trace minimization
\begin{equation}
\label{eq:reduced_KNI}
\begin{aligned}
\min \quad
&\mathrm{tr}(\symbfit Z^\top\symbfit M_\lambda\symbfit Z)
\\[2mm]
\textnormal{s.~t.} \quad
&\symbfit Z^\top\symbfit Z=\symbfit I_{\hat{d}}, \\
&\symbfit{Z}\in\mathbb{R}^{n_{\mathrm{a}}\times\hat{d}}.
\end{aligned}
\end{equation}

The following theorem summarizes the two-step reduction of the kernel-based target-normalized integration problem.

\begin{theorem}[Reduction of the Kernel-based Target-Normalized Integration Problem]
\label{thm:kernel-reduction}
Assume that $\hat{d}\le n_{\mathrm{a}}$.
The problem in Eq.~\eqref{eq:kernel-formulation} is reduced to the constrained trace minimization problem in Eq.~\eqref{eq:reduced_KNI}.
Moreover, for fixed $\symbfit Z$, the optimal integration function is given by Eq.~\eqref{eq:gk-vector-solution}.
\end{theorem}

The proof is provided in~\ref{app:proof-kernel-reduction}.
Since each $\symbfit \Sigma_k$ is symmetric positive semidefinite and $\lambda>0$, $\symbfit S_k$ in Eq.~\eqref{eq:Sk-Mlambda-definition} is symmetric positive definite.
Therefore, $\symbfit M_\lambda$ is also symmetric positive definite.
By the Ky Fan minimum principle~\citep{ky-fan}, a globally optimal solution to Eq.~\eqref{eq:reduced_KNI} is obtained from the eigenvalue decomposition of $\symbfit M_\lambda$.
Let $\symbfit U_{\hat{d}}\in\mathbb R^{n_{\mathrm{a}}\times\hat{d}}$ be the matrix whose columns are the eigenvectors of $\symbfit M_\lambda$ corresponding to the $\hat{d}$ smallest eigenvalues.
Then, for any $\symbfit O\in\mathbf O(\hat{d})$,
\begin{equation}
\label{eq:Zstar}
\symbfit Z^\star=\symbfit U_{\hat{d}}\symbfit O
\end{equation}
is a globally optimal solution to Eq.~\eqref{eq:reduced_KNI}.
The arbitrary orthogonal matrix $\symbfit O$ reflects the rotational indeterminacy of the target representation.
When the $\hat{d}$-th smallest eigenvalue is separated from the $(\hat{d}+1)$-th smallest eigenvalue, the optimal eigenspace is unique, and this rotational indeterminacy is the only remaining ambiguity.
Using $\symbfit Z^\star$ in Eq.~\eqref{eq:gk-vector-solution}, we obtain
\begin{equation}
\label{eq:nki-optimal-function}
g_k^\star(\symbfit x)=\symbfit\kappa_k(\symbfit x)\symbfit S_k\symbfit Z^\star .
\end{equation}

Algorithm~\ref{alg:nki} summarizes the resulting procedure for constructing KTI functions.

\begin{algorithm}[t]
\caption{Construction of KTI functions}
\label{alg:nki}
\begin{algorithmic}[1]
\Require Intermediate representations of the anchor dataset $\{\tilde{\symbfit A}^{(k)}\}_{k=1}^K$, kernels $\{\kappa_k\}_{k=1}^K$, regularization parameter $\lambda>0$
\Ensure Target representation $\symbfit Z^\star$ and integration functions $\{g_k^\star\}_{k=1}^K$

\For{$k=1,\ldots,K$}
    \State Construct the kernel matrix $\symbfit \Sigma_k\in\mathbb R^{n_{\mathrm{a}}\times n_{\mathrm{a}}}$ according to Eq.~\eqref{eq:kernel-matrix-definition}.
    \State Compute $\symbfit S_k\gets(\symbfit \Sigma_k+\lambda\symbfit I_{n_{\mathrm{a}}})^{-1}$.
\EndFor
\State Construct $\symbfit M_\lambda\gets\lambda\sum_{k=1}^K \symbfit S_k$.
\State Compute $\symbfit Z^\star$ by solving the corresponding reduced problem in Eq.~\eqref{eq:reduced_KNI}.
\For{$k=1,\ldots,K$}
    \State Define $g_k^\star(\symbfit x)\gets\symbfit\kappa_k(\symbfit x)\symbfit S_k\symbfit Z^\star$.
\EndFor
\end{algorithmic}
\end{algorithm}


\subsection{Computational Complexity}
\label{sec:complexity}

We analyze computational complexity under the assumptions that each $\tilde{\symbfit A}^{(k)}$ has full column rank and that the intermediate dimension is common across parties, $\tilde d(k)=\tilde d=\hat d$.
Here, integration-function construction refers to constructing all party-wise integration functions $\{g_k\}_{k=1}^K$, whereas collaboration-representation construction refers to computing $g_k(\symbfit x)$ for one new sample $\symbfit x\in\mathbb R^{1\times\tilde d}$ of party $k$.
For existing methods, including MPP, GEP, and ODC, we use the computational complexities reported in prior work~\cite{DC-odc}.

For LTI, the dominant cost of integration-function construction is the singular value decomposition of $\symbfit W_Q\in\mathbb R^{n_{\mathrm{a}}\times K\tilde d}$, which is
\[
\mathcal{O}\!\left(\min\{n_{\mathrm{a}}(K\tilde d)^2,\;n_{\mathrm{a}}^2K\tilde d\}\right).
\]
This order coincides with MPP and GEP.
For collaboration-representation construction, LTI applies the linear map $\symbfit x\mapsto \symbfit x\symbfit G^{(k)}$, resulting in $\mathcal{O}(\tilde d^2)$ under $\tilde d=\hat d$.

For KTI, constructing all kernel matrices costs $\mathcal{O}(Kn_{\mathrm{a}}^2\tilde d)$ for standard kernel evaluation.
Computing the matrices $\symbfit S_k$ by solving $(\symbfit\Sigma_k+\lambda\symbfit I_{n_{\mathrm{a}}})\symbfit S_k=\symbfit I_{n_{\mathrm{a}}}$ for all $k\in[K]$ costs $\mathcal{O}(Kn_{\mathrm{a}}^3)$.
Solving the reduced eigenvalue problem for $\symbfit Z^\star$ costs $\mathcal{O}(n_{\mathrm{a}}^3)$.
Therefore, the dominant complexity of integration-function construction for KTI is
\[
\mathcal{O}(Kn_{\mathrm{a}}^3+Kn_{\mathrm{a}}^2\tilde d).
\]
Once $\symbfit S_k$ and $\symbfit Z^\star$ are computed, constructing the collaboration representation of one new sample requires evaluating $\symbfit\kappa_k(\symbfit x)\symbfit S_k\symbfit Z^\star$, whose dominant cost is $\mathcal{O}(n_{\mathrm{a}}\tilde d)$.

Table~\ref{tab:complexity} summarizes these results.

\begin{table}[t]
\centering
\caption{Computational complexity of integration methods.}
\label{tab:complexity}
\begin{tabularx}{\linewidth}{lXX}
\toprule
Method & Integration-function construction & Collaboration-representation construction \\
\midrule
MPP~\cite{DCframework2}
  & $\mathcal{O}\!\left(\min\{K^2 n_{\mathrm{a}}\tilde d^2,\,
    K n_{\mathrm{a}}^2\tilde d\}\right)$
  & $\mathcal{O}(\tilde d^2)$ \\
GEP~\cite{KawakamiDC}
  & $\mathcal{O}\!\left(\min\{K^2 n_{\mathrm{a}}\tilde d^2,\,
    K n_{\mathrm{a}}^2\tilde d\}\right)$
  & $\mathcal{O}(\tilde d^2)$ \\
ODC~\cite{DC-odc}
  & $\mathcal{O}(K n_{\mathrm{a}}\tilde d^2)$
  & $\mathcal{O}(\tilde d^2)$ \\
\midrule
LTI
  & $\mathcal{O}\!\left(\min\{K^2 n_{\mathrm{a}}\tilde d^2,\,
    K n_{\mathrm{a}}^2\tilde d\}\right)$
  & $\mathcal{O}(\tilde d^2)$ \\
KTI
  & $\mathcal{O}(K n_{\mathrm{a}}^3 + K n_{\mathrm{a}}^2\tilde d)$
  & $\mathcal{O}(n_{\mathrm{a}}\tilde d)$ \\
\bottomrule
\end{tabularx}
\end{table}

\subsection{Graph-Based Regularization}
\label{sec:graph-embedding}

The target representation $\symbfit Z^\star$ obtained by KTI aligns the intermediate representations of the anchor dataset across parties, but it does not explicitly exploit relational structure among anchor data instances.
To incorporate such structure, we introduce graph-based regularization into the estimation of the target representation.
Following the graph embedding framework~\cite{GraphEmbeddingExtensions}, we use two graphs: an intrinsic graph and a penalty graph. The intrinsic graph specifies pairs of anchor data instances that should be placed close to each other in the target representation, whereas the penalty graph specifies pairs that should be separated.

For each pair $(i,i')\in[n_{\mathrm a}]\times[n_{\mathrm a}]$ of anchor instances, let $w^{(B)}_{ii'}$ and $w^{(C)}_{ii'}$ be the weights of the intrinsic graph and the penalty graph, respectively.
Specific constructions of these weights are given in Section~\ref{sec:graph-construction}.
The corresponding weight matrices are defined as follows:
\begin{equation}
\label{eq:graph-weight-matrices}
\symbfit W^{(B)}:=[w^{(B)}_{ii'}]_{(i,i')\in[n_{\mathrm a}]\times[n_{\mathrm a}]}\in\mathbb R^{n_{\mathrm a}\times n_{\mathrm a}},
\quad
\symbfit W^{(C)}:=[w^{(C)}_{ii'}]_{(i,i')\in[n_{\mathrm a}]\times[n_{\mathrm a}]}\in\mathbb R^{n_{\mathrm a}\times n_{\mathrm a}}.
\end{equation}
Let $\symbfit D^{(B)}$ and $\symbfit D^{(C)}$ be the corresponding degree matrices, that is, diagonal matrices whose diagonal entries are defined by
\[
(\symbfit D^{(B)})_{ii}
=
\sum_{i'=1}^{n_{\mathrm{a}}} w^{(B)}_{ii'},
\qquad
(\symbfit D^{(C)})_{ii}
=
\sum_{i'=1}^{n_{\mathrm{a}}} w^{(C)}_{ii'}.
\]
Using these matrices, we define the intrinsic-graph and penalty-graph Laplacians as follows:
\begin{equation}
\label{eq:graph-laplacian-definition}
\symbfit{B}:=\symbfit{D}^{(B)}-\symbfit{W}^{(B)}\in\mathbb{R}^{n_{\mathrm{a}}\times n_{\mathrm{a}}},
\qquad
\symbfit{C}:=\symbfit{D}^{(C)}-\symbfit{W}^{(C)} \in\mathbb{R}^{n_{\mathrm{a}}\times n_{\mathrm{a}}}.
\end{equation}
The following standard identity for graph Laplacians, as used in Laplacian eigenmaps~\cite{LaplacianEigenmaps}, relates the Laplacian quadratic forms to pairwise distances:
\begin{equation}
\label{eq:laplacian-distance-identity}
\begin{aligned}
\mathrm{tr}(\symbfit{Z}^\top \symbfit{B} \symbfit{Z})
&=\frac{1}{2}\sum_{i,i'} w^{(B)}_{ii'}\|\symbfit{Z}_{i,:}-\symbfit{Z}_{i',:}\|_2^2, \\
\mathrm{tr}(\symbfit{Z}^\top \symbfit{C} \symbfit{Z})
&=\frac{1}{2}\sum_{i,i'} w^{(C)}_{ii'}\|\symbfit{Z}_{i,:}-\symbfit{Z}_{i',:}\|_2^2 .
\end{aligned}
\end{equation}
Starting from the reduced KTI problem in Eq.~\eqref{eq:reduced_KNI}, we incorporate graph embedding as follows:
\begin{equation}
\label{eq:kernel-graph-reduced-problem}
\begin{aligned}
\min \quad
&\mathrm{tr}\!\left(
\symbfit Z^\top(\symbfit M_\lambda+\mu\symbfit B)\symbfit Z
\right)
\\[2mm]
\text{s.~t.} \quad
&\symbfit Z^\top\symbfit C\symbfit Z=\symbfit I_{\hat{d}}, \\
&\symbfit{Z}\in\mathbb{R}^{n_{\mathrm{a}}\times\hat{d}}.
\end{aligned}
\end{equation}
Here, $\mu\ge0$ controls the strength of graph regularization.
Compared with Eq.~\eqref{eq:reduced_KNI}, Eq.~\eqref{eq:kernel-graph-reduced-problem} adds the intrinsic-graph Laplacian to the objective and uses the penalty-graph Laplacian in the constraint.
By Eq.~\eqref{eq:laplacian-distance-identity}, the added objective term penalizes weighted pairwise distances between rows of $\symbfit Z$ according to $w^{(B)}_{ii'}$; hence, pairs with larger intrinsic-graph weights are encouraged to be closer.
In contrast, the penalty-graph constraint fixes the total weighted pairwise distance associated with $w^{(C)}_{ii'}$, since $\mathrm{tr}(\symbfit Z^\top\symbfit C\symbfit Z)=\hat{d}$.

By the Lagrange multiplier method, Eq.~\eqref{eq:kernel-graph-reduced-problem} leads to the generalized eigenvalue problem
\begin{equation}
\label{eq:kernel-graph-gep}
(\symbfit{M}_\lambda+\mu\symbfit{B})\symbfit{u}
=
\gamma\,\symbfit{C}\symbfit{u} .
\end{equation}
If $\symbfit C$ is singular, we replace it with $\symbfit C_\varepsilon=\symbfit C+\varepsilon\symbfit I_{n_{\mathrm{a}}}$ for $\varepsilon>0$ and solve the corresponding regularized problem.
As in the derivation of Theorem~\ref{thm:kernel-reduction}, a globally optimal target representation $\symbfit Z^\star$ is obtained from the $\hat{d}$ generalized eigenvectors of Eq.~\eqref{eq:kernel-graph-gep} associated with the smallest generalized eigenvalues (with $\symbfit C$ replaced by $\symbfit C_\varepsilon$ in the regularized case), up to the same rotational indeterminacy.
Because the graph terms modify only the reduced problem over $\symbfit Z$, the fixed-$\symbfit Z$ optimization over $g_k$ is unchanged.
Thus, using $\symbfit Z^\star$ in Eq.~\eqref{eq:gk-vector-solution} yields the graph-regularized integration function.

\subsection{Graph Construction}
\label{sec:graph-construction}

This subsection describes how to construct the intrinsic-graph weights $w_{ii'}^{(B)}$ and penalty-graph weights $w_{ii'}^{(C)}$ used in Eq.~\eqref{eq:graph-weight-matrices}.
Since the analyst has access only to the intermediate representations of the anchor dataset, graph weights are constructed party-wise from $\tilde{\symbfit A}^{(k)}$ and then aggregated across parties.

For each party $k\in[K]$ and each pair $i,i'\in[n_{\mathrm{a}}]$, let $\hat w_{ii'}^{(Bk)}$ and $\hat w_{ii'}^{(Ck)}$ denote per-party raw weights.
To make the resulting Laplacians symmetric, we first symmetrize these weights:
\begin{equation}
\label{eq:graph-weight-symmetrization}
w_{ii'}^{(Bk)}\coloneqq \frac{1}{2}\!\left(\hat{w}_{ii'}^{(Bk)}+\hat{w}_{i'i}^{(Bk)}\right),
\qquad
w_{ii'}^{(Ck)}\coloneqq \frac{1}{2}\!\left(\hat{w}_{ii'}^{(Ck)}+\hat{w}_{i'i}^{(Ck)}\right).
\end{equation}
The party-wise weights are then averaged as
\begin{equation}
\label{eq:graph-weight-aggregation}
w_{ii'}^{(B)}\coloneqq \frac{1}{K}\sum_{k=1}^{K} w_{ii'}^{(Bk)},
\qquad
w_{ii'}^{(C)}\coloneqq \frac{1}{K}\sum_{k=1}^{K} w_{ii'}^{(Ck)} .
\end{equation}
These aggregated weights define $\symbfit B$ and $\symbfit C$ through Eq.~\eqref{eq:graph-laplacian-definition}.

Let $\mathcal{N}_{k_{\mathrm{nn}}}^{(k)}(i)$ denote the $k_{\mathrm{nn}}$-nearest-neighbor set of anchor instance $i$ in $\tilde{\symbfit A}^{(k)}$.
When target variables are assigned to anchor instances, we denote them by $y_{\mathrm{a}i}$.
The graph construction uses only the intermediate representations and target variables of anchor instances, and does not require access to original private datasets.

In the main text, we define two intrinsic-graph choices used in the experiments: the geometric Laplacian (GL) and the target-similarity Laplacian (TSL).
When no penalty graph is specified, we set $\symbfit C=\symbfit I_{n_{\mathrm{a}}}$.
Other graph choices, including target-dissimilarity and regression variants, are described in~\ref{app:graph-construction}.

\paragraph{Geometric Laplacian (GL)}
GL uses nearest-neighbor relationships in the intermediate representations:
\begin{equation}
\label{eq:gl-weight}
\hat w_{ii'}^{(Bk)}=
\begin{cases}
1 & \text{if } i'\in\mathcal N_{k_{\mathrm{nn}}}^{(k)}(i),\\
0 & \text{otherwise.}
\end{cases}
\end{equation}

\paragraph{Target-Similarity Laplacian (TSL)}
For classification tasks, TSL keeps nearest-neighbor anchor instances with the same target variable close:
\begin{equation}
\label{eq:tsl-class-weight}
\hat w_{ii'}^{(Bk)}=
\begin{cases}
1 & \text{if } i'\in\mathcal N_{k_{\mathrm{nn}}}^{(k)}(i),\ y_{\mathrm{a}i}=y_{\mathrm{a}i'},\\
0 & \text{otherwise.}
\end{cases}
\end{equation}

\subsection{Centering Constraint}
\label{sec:centering}

The constraint $\symbfit{Z}^\top\symbfit{C}\symbfit{Z}=\symbfit{I}_{\hat d}$ in Eq.~\eqref{eq:kernel-graph-reduced-problem} fixes the scale and orthogonality of the target representation but does not directly constrain its column mean $\symbfit{1}_{n_{\mathrm a}}^\top\symbfit Z$.
Consequently, a redundant constant-direction component can remain in the solution.
This component corresponds to a degree of freedom in the origin of the embedding and may reduce the interpretability and comparability of the target representation.

To remove this component, we impose the following centering constraint:
\begin{equation}
\label{eq:centering}
\symbfit 1_{n_{\mathrm{a}}}^\top\symbfit Z=\symbfit 0_{1\times\hat{d}},
\end{equation}
where $\symbfit 0_{1\times\hat{d}}=[0,\ldots,0]\in\mathbb{R}^{1\times\hat{d}}$.

For the centered formulation, we assume that $\hat d\le n_{\mathrm a}-1$ because the feasible subspace orthogonal to $\symbfit 1_{n_{\mathrm a}}$ has dimension $n_{\mathrm a}-1$.
The problem in Eq.~\eqref{eq:kernel-graph-reduced-problem} is then extended as follows:
\begin{equation}
\label{eq:kernel-mean-zero-reduced-problem}
\begin{aligned}
\min \quad
&\mathrm{tr}\!\left(
\symbfit Z^\top(\symbfit M_\lambda+\mu\symbfit B)\symbfit Z
\right)
\\[2mm]
\text{s.~t.} \quad
&\symbfit Z^\top\symbfit C\symbfit Z=\symbfit I_{\hat d},\\
&\symbfit 1_{n_{\mathrm a}}^\top\symbfit Z=\symbfit 0_{1\times\hat d},\\
&\symbfit{Z}\in\mathbb{R}^{n_{\mathrm{a}}\times\hat d}.
\end{aligned}
\end{equation}

The second constraint in Eq.~\eqref{eq:kernel-mean-zero-reduced-problem}, namely Eq.~\eqref{eq:centering}, is equivalent to requiring that each column of $\symbfit Z$ belongs to
\begin{equation}
\label{eq:centered-subspace}
\mathcal{U}=\left\{\symbfit{x}\in\mathbb{R}^{n_{\mathrm{a}}}\mid\symbfit{1}_{n_{\mathrm{a}}}^\top\symbfit{x}=0\right\} .
\end{equation}
Let $\symbfit{T}\in\mathbb{R}^{n_{\mathrm{a}}\times(n_{\mathrm{a}}-1)}$ be a matrix whose columns form an orthonormal basis for $\mathcal{U}$, i.e., $\symbfit{1}_{n_{\mathrm{a}}}^\top\symbfit{T}=\symbfit{0}$ and $\symbfit{T}^\top\symbfit{T}=\symbfit{I}_{n_{\mathrm{a}}-1}$.
Then any $\symbfit Z$ feasible for Eq.~\eqref{eq:kernel-mean-zero-reduced-problem} can be represented as
\begin{equation}
\label{eq:Z-TY}
\symbfit Z=\symbfit T\symbfit Y,
\qquad
\symbfit Y\in\mathbb R^{(n_{\mathrm a}-1)\times\hat d} .
\end{equation}
Substituting Eq.~\eqref{eq:Z-TY} into Eq.~\eqref{eq:kernel-mean-zero-reduced-problem}, we obtain
\begin{equation}
\label{eq:reduced-Y}
\begin{aligned}
\min \quad
&\mathrm{tr}\!\left(
\symbfit Y^\top\tilde{\symbfit M}\symbfit Y
\right)
\\[2mm]
\text{s.~t.} \quad
&\symbfit Y^\top\tilde{\symbfit C}\symbfit Y=\symbfit I_{\hat{d}},\\
&\symbfit{Y}\in\mathbb{R}^{(n_{\mathrm{a}}-1)\times\hat{d}},
\end{aligned}
\end{equation}
where the reduced matrices are defined as
$\tilde{\symbfit{M}}=\symbfit{T}^\top(\symbfit{M}_\lambda+\mu\symbfit{B})\symbfit{T},\;
\tilde{\symbfit{C}}=\symbfit{T}^\top\symbfit{C}\symbfit{T}.$

The reduced problem in Eq.~\eqref{eq:reduced-Y} is solved as a generalized eigenvalue problem, with the same $\varepsilon$-regularization applied when necessary.
Taking the generalized eigenvectors corresponding to the $\hat{d}$ smallest generalized eigenvalues gives an optimal $\symbfit Y^\star$.
The centered target representation is then recovered as $\symbfit Z^\star=\symbfit T\symbfit Y^\star$, and the centered integration function $g_k^\star$ is obtained from Eq.~\eqref{eq:gk-vector-solution}.
We refer to the centered variant without graph regularization as KTI+Center and the centered graph-regularized variant as KTI+Graph+Center.

The graph-regularized and centered variants have the same leading asymptotic complexity as KTI.
They add only preprocessing terms and replace the reduced eigenvalue problem with a generalized eigenvalue problem of the same size.
The per-sample collaboration-representation construction cost is also unchanged, because new instances are still transformed using Eq.~\eqref{eq:nki-optimal-function}.

\paragraph{Illustrative visualization of collaboration representations}
To illustrate the effects of graph-based regularization and the centering constraint on collaboration representations, Fig.~\ref{fig:graph-center-illustration} shows a toy example with three classes, where color of each point indicates its target label.
Fig.~\ref{fig:graph-center-illustration}(a) shows an example of intermediate representations of instances in the anchor dataset for one party.
Figs.~\ref{fig:graph-center-illustration}(b)--(d) show collaboration representations obtained by KTI, KTI+TSL, and KTI+TSL+Center, respectively.
KTI nonlinearly integrates intermediate representations into a common feature space, but it does not directly use neighborhood relationships or target-variable information; therefore, geometric or class structures may not be sufficiently preserved in the resulting collaboration representations.
In contrast, Fig.~\ref{fig:graph-center-illustration}(c) shows that adding TSL as graph regularization tends to place neighboring instances with the same target label close to each other in the collaboration representation.
Furthermore, Fig.~\ref{fig:graph-center-illustration}(d) shows that the centering constraint removes a redundant constant-direction component.

\begin{figure*}[t]
  \centering
  \begin{subfigure}{0.24\linewidth}
    \centering
    \includegraphics[width=\linewidth]{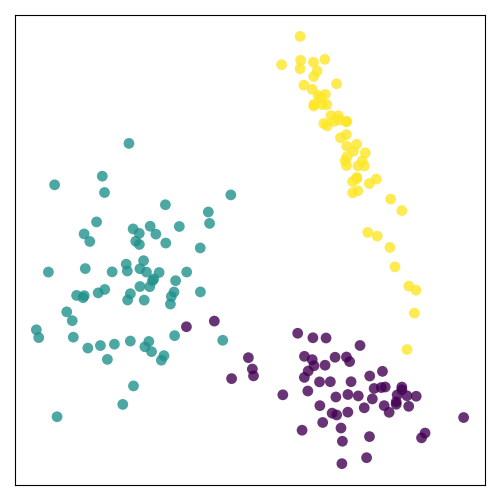}
    \caption{Intermediate}
  \end{subfigure}
  \hfill
  \begin{subfigure}{0.24\linewidth}
    \centering
    \includegraphics[width=\linewidth]{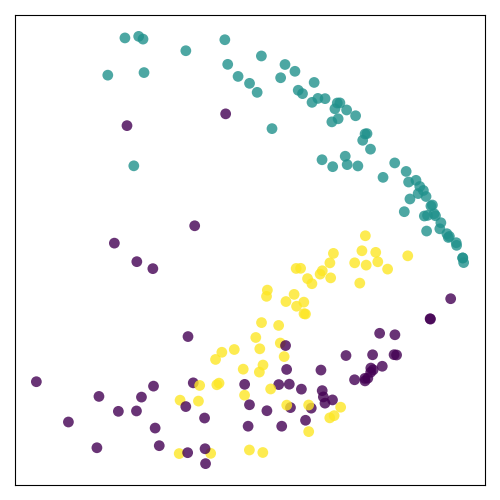}
    \caption{KTI}
  \end{subfigure}
  \hfill
  \begin{subfigure}{0.24\linewidth}
    \centering
    \includegraphics[width=\linewidth]{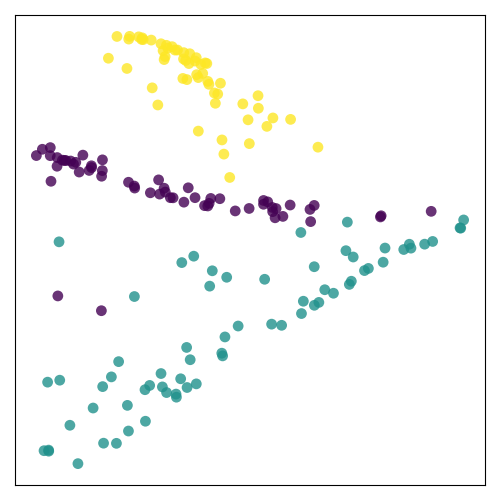}
    \caption{KTI+TSL}
  \end{subfigure}
  \hfill
  \begin{subfigure}{0.24\linewidth}
    \centering
    \includegraphics[width=\linewidth]{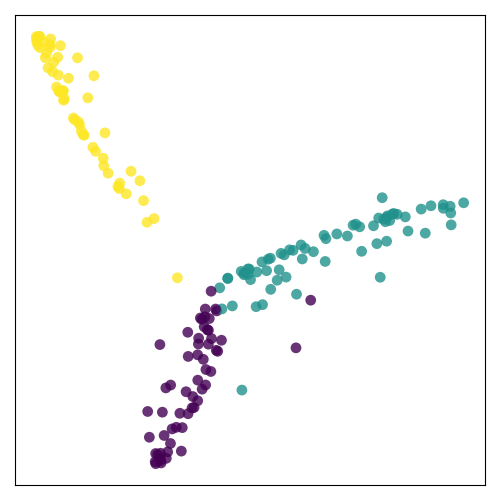}
    \caption{KTI+TSL+Center}
  \end{subfigure}
  \caption{Illustrative visualization of KTI-based collaboration representations.}
  \label{fig:graph-center-illustration}
\end{figure*}

\section{Experiments}
\label{sec:experiments}

%

This section evaluates the proposed method through four research questions.
\begin{itemize}
    \item \textbf{RQ1}: Does KTI improve classification accuracy compared with existing DC integration methods?
    \item \textbf{RQ2}: How do the proposed extensions, namely graph regularization and centering, contribute to classification accuracy?
    \item \textbf{RQ3}: How do anchor size and anchor quality affect classification accuracy and computation time?
    \item \textbf{RQ4}: How do dimensionality reduction methods and intermediate dimensions affect the trade-off between classification accuracy and reconstruction risk?
\end{itemize}
The following experiments focus on classification tasks.
As a supplementary evaluation,~\ref{app:regression-evaluation} reports regression experiments on two additional datasets.

\subsection{Datasets and Common Experimental Setup}
\label{subsec:experimental_setup}

\paragraph{Datasets and partitioning}
Unless otherwise stated, all experiments used the default settings summarized in Table~\ref{tab:default-settings}.
Settings that differed from the defaults were described at the beginning of each experiment.
We used MNIST~\cite{MNIST} and Fashion-MNIST~\cite{FASHION}, both of which are 10-class classification datasets consisting of $28\times28$ grayscale images with input dimension $d=784$.
Input features were normalized to $[0,1]$ by dividing each pixel value by 255.
Let $N$ denote the total number of training instances.
In each trial, the training instances were randomly partitioned into $K$ parties with equal size, so that each party had $n(k)=N/K$ training instances for all $k\in[K]$.
The test dataset was shared across all parties and had size $N$.

\paragraph{Dimensionality reduction and anchor dataset}
Unless otherwise stated, each party used UMAP~\cite{UMAP} for dimensionality reduction, and the intermediate dimension was common across parties, i.e., $\tilde d(k)=\tilde d$ for all $k\in[K]$.
To introduce heterogeneity among parties, we varied the UMAP hyperparameters across parties.
The distance metric was assigned by party index: parties with $k \bmod 3 = 0$, $1$, and $2$ used \texttt{correlation}, \texttt{cosine}, and \texttt{euclidean}, respectively.
The \texttt{n\_neighbors} parameter was sampled uniformly from $\{2,\ldots,7\}$, and \texttt{min\_dist} was sampled uniformly from $[0.0, 0.8)$.
For all integration methods, the integrated dimension was set to $\hat d=\tilde d$.

Following the prior report that using SMOTE for anchor dataset generation can improve DC analysis performance~\cite{DC-SMOTE}, we generated anchor datasets using SMOTE.
Here, $n_{\mathrm a}^{\mathrm{smote}}$ denotes the number of real instances used as SMOTE source instances, and $n_{\mathrm a}$ denotes the total anchor dataset size.
The source instances were disjoint from the training and test data and were included in the anchor dataset.
The anchor dataset was class-balanced, with $n_{\mathrm a}/10$ instances per class.
SMOTE was applied only between source instances sharing the same target variable, and each generated anchor instance inherited that target variable.
The neighborhood size for SMOTE was set to $k_{\mathrm{nn}}$.

\paragraph{Evaluation protocol and model settings}
The random seed for each trial was set to the trial index $r$ and was applied to data partitioning, UMAP embedding, SMOTE-based anchor generation, and Random Forest classifier training.
We conducted independent trials for each experimental condition; the number of trials $N_{\mathrm{seed}}$ is specified separately in each experiment.
Classification accuracy was evaluated as the mean test accuracy across all parties, where each party trained a classifier on its own collaboration representations.
The following methods and baselines were used in the experiments, with the specific subset stated in each experiment:
\begin{itemize}
    \item \textbf{Central}: All party data are centrally aggregated to train a single model; this method ignores privacy and serves as an idealized upper bound.
    \item \textbf{Local}: Each party trains a separate model on its own data only; this serves as a non-collaborative baseline.
    \item \textbf{DC-MPP}: Integration based on the minimum perturbation problem~\cite{DCframework2}.
    \item \textbf{DC-GEP}: Integration based on the generalized eigenvalue problem~\cite{KawakamiDC}.
    \item \textbf{DC-ODC}: Integration based on the orthogonal Procrustes problem~\cite{DC-odc}.
    \item \textbf{DC-LTI}: The proposed linear target-normalized integration based on Eq.~\eqref{eq:linear-formulation}.
    \item \textbf{DC-KTI}: The proposed kernel-based target-normalized integration based on Eq.~\eqref{eq:kernel-formulation}.
\end{itemize}
We used \texttt{RandomForestClassifier} with default settings from scikit-learn (version 1.6.1)~\cite{sklearn}, and results were reported as the mean and 95\% confidence interval over independent trials.
As the integration kernel $\kappa_k$ in KTI, we used the RBF kernel
\[
\kappa_k(\symbfit x,\symbfit x')
=
\exp\!\left(
-\gamma_k \|\symbfit x-\symbfit x'\|_2^2
\right).
\]
The RBF kernel parameter was shared across all parties, i.e., $\gamma_k=\gamma$ for all $k\in[K]$.
The regularization parameter $\lambda$, the graph regularization parameter $\mu$, and the RBF kernel parameter $\gamma$ were fixed throughout all trials without validation-based tuning.

\paragraph{Implementation environment}
All experiments were conducted on a MacBook Air with Apple M4 (32\,GB memory, macOS 15.7.2, arm64).
All methods were implemented in Python 3.9.6 with NumPy 1.26.4, pandas 2.3.2, SciPy 1.13.1, scikit-learn 1.6.1, and umap-learn 0.5.9.post2.
To eliminate variability from parallelism and ensure a fair runtime comparison across methods, BLAS was fixed to single-threaded execution by setting \texttt{OMP\_NUM\_THREADS}, \texttt{OPENBLAS\_NUM\_THREADS}, and \texttt{MKL\_NUM\_THREADS} to 1.

\begin{table}[t]
\centering
\caption{Default experimental settings.}
\label{tab:default-settings}
\setlength{\tabcolsep}{5pt}
\renewcommand{\arraystretch}{1.05}
\begin{tabular}{ll}
\toprule
Item & Default setting \\
\midrule
Datasets & MNIST, Fashion-MNIST \\
Task & 10-class classification \\
Feature dimension & $d=784$ \\
Feature normalization & Pixel values scaled to $[0,1]$ \\
Number of parties & $K=10$ \\
Total training instances & $N=1000$ \\
Training instances per party & $n(k)=N/K=100$ \\
Test instances & $N$ common instances \\
Dimensionality reduction & UMAP \\
Intermediate dimension & $\tilde d=10$ \\
Integrated dimension & $\hat d=\tilde d$ \\
Anchor size & $n_{\mathrm a}=1000$ \\
SMOTE source instances & $n_{\mathrm a}^{\mathrm{smote}}=100$ \\
Anchor class balance & $n_{\mathrm a}/10$ instances per class \\
Classifier & RandomForestClassifier default settings \\
Kernel & RBF kernel with $\gamma=1$ \\
Regularization & $\lambda=1$, $\mu=1$ \\
SMOTE/graph neighborhood size & $k_{\mathrm{nn}}=10$ \\
\bottomrule
\end{tabular}
\end{table}

\subsection{Comparison with Existing Integration Methods}
\label{subsec:rq1_existing_methods}

We first evaluated whether the proposed DC-KTI outperformed existing linear integration methods.

\paragraph{Experimental setup}
The number of parties $K$ was varied over $\mathcal{K}=\{2,4,8,16,32,64\}$, and the number of trials per condition was $N_{\mathrm{seed}}=100$.
For each value of $K$, the number of training instances per party was fixed at $n(k)=100$; hence, the total number of training instances was $N=100K$.
This design fixed each party's own data budget across $K$, so that we could evaluate whether classification accuracy improved as the number of parties, and thus the amount of data usable for the collaborative analysis, increased.
We compared Central, Local, DC-MPP, DC-GEP, DC-ODC, DC-LTI, and DC-KTI.

\paragraph{Results}
Fig.~\ref{fig:num-existing} shows the classification accuracy as a function of the number of parties $K$.
Central was also evaluated as an idealized upper bound but was excluded from Fig.~\ref{fig:num-existing} to preserve readability.
The mean accuracy of Central ranged from 0.680 to 0.941 on MNIST and from 0.656 to 0.844 on Fashion-MNIST as $K$ increased from 2 to 64.

\begin{figure}[t]
  \centering
  \begin{subfigure}{0.49\linewidth}
    \centering
    \includegraphics[width=\linewidth]{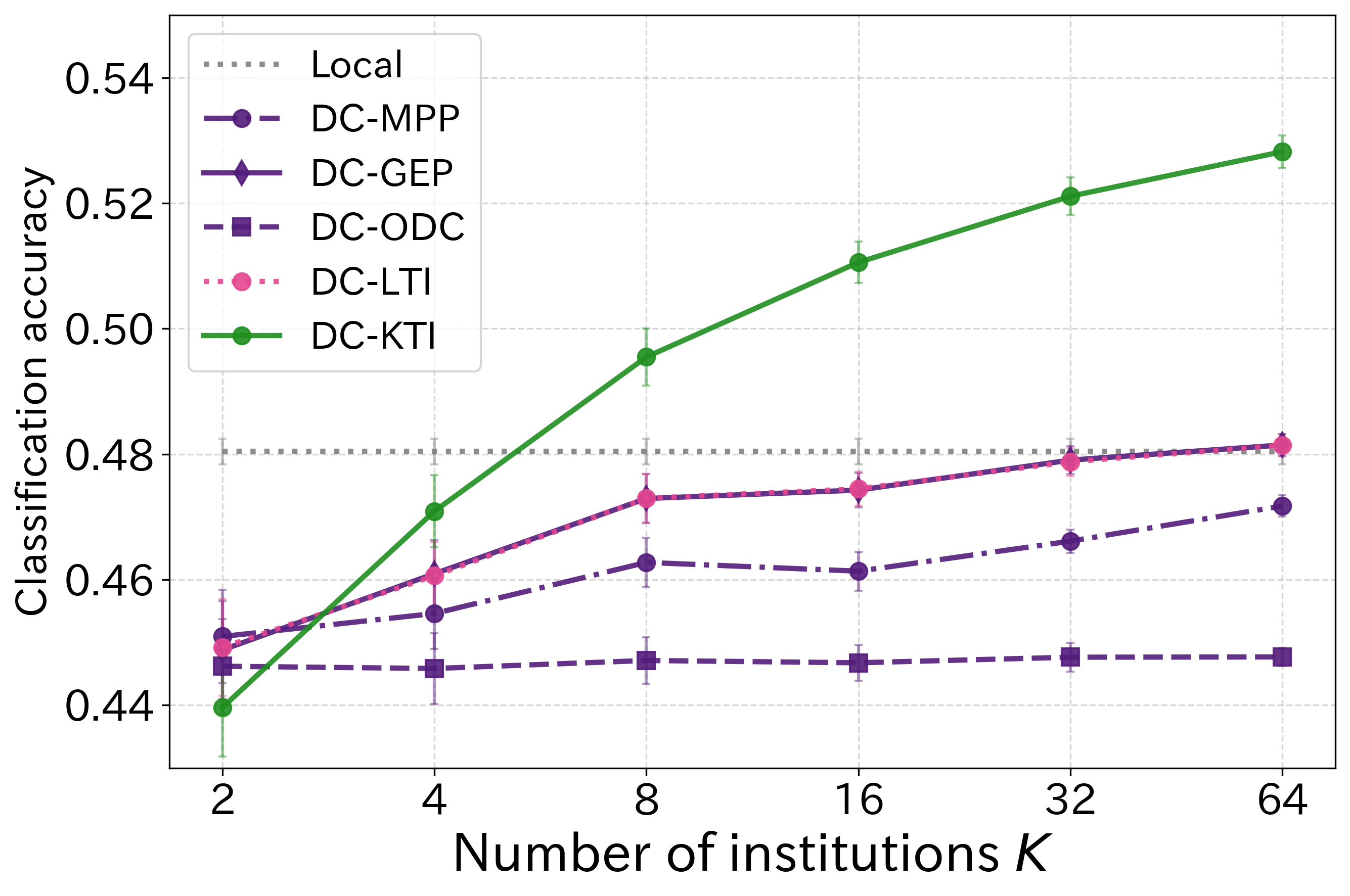}
    \caption{MNIST}
  \end{subfigure}
  \hfill
  \begin{subfigure}{0.49\linewidth}
    \centering
    \includegraphics[width=\linewidth]{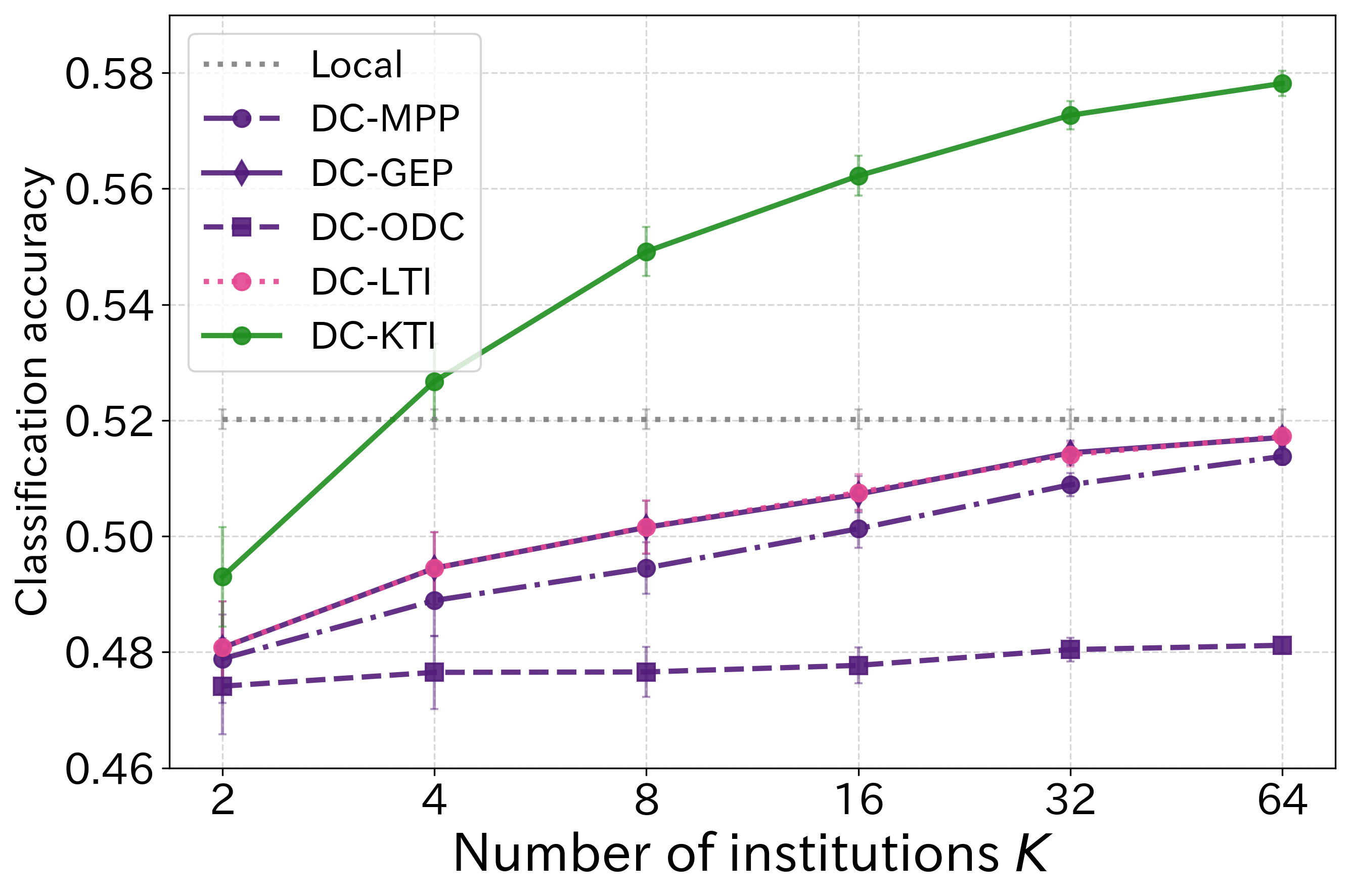}
    \caption{Fashion-MNIST}
  \end{subfigure}
  \caption{Classification accuracy for integration methods and baselines in RQ1.}
  \label{fig:num-existing}
\end{figure}

As shown in Fig.~\ref{fig:num-existing}, DC-KTI generally achieved higher classification accuracy than the existing linear integration methods on both MNIST and Fashion-MNIST in most conditions, with the advantage becoming more pronounced at larger $K$.
In contrast, linear integration methods (DC-MPP, DC-GEP, DC-ODC, and DC-LTI) underperformed DC-KTI in most conditions.

Among the linear integration methods, DC-LTI and DC-GEP performed comparably and tended to rank highest in accuracy, followed by DC-MPP and DC-ODC.
The fact that linear integration methods frequently underperformed Local is likely because these methods cannot sufficiently align the nonlinearly transformed intermediate representations produced by UMAP, and the cost of misalignment and information loss from dimensionality reduction outweighs the benefit of integration.
These results suggested that, when nonlinear dimensionality reduction was used, employing an integration function that accounts for the nonlinearity of intermediate representations was more effective than directly applying linear integration methods.

\subsection{Effect of Graph Regularization and Centering}
\label{subsec:rq2_extensions}

This experiment evaluated the contributions of graph regularization and the centering constraint.

\paragraph{Experimental setup}
As in RQ1, the number of parties $K$ was varied over $\mathcal{K}=\{2,4,8,16,32,64\}$, and the number of trials per condition was $N_{\mathrm{seed}}=100$.
For each value of $K$, the number of training instances per party was fixed at $n(k)=100$; hence, the total number of training instances was $N=100K$.
From Fig.~\ref{fig:num_2} onward, we omit the prefix ``DC-'' from method names for readability.
We compared KTI with the following variants:
\begin{itemize}
    \item KTI+Center: KTI with the centering constraint in Eq.~\eqref{eq:centering};
    \item KTI+GL: KTI with graph regularization using GL weights in Eq.~\eqref{eq:gl-weight};
    \item KTI+TSL: KTI with graph regularization using TSL weights in Eq.~\eqref{eq:tsl-class-weight};
    \item KTI+GL+Center and KTI+TSL+Center: the corresponding centered variants.
\end{itemize}
For graph construction, the neighborhood size was also set to $k_{\mathrm{nn}}$.
For the graph-regularized variants, Eq.~\eqref{eq:kernel-graph-reduced-problem} was solved after replacing $\symbfit M_\lambda$ and $\symbfit B$ with $\symbfit M_\lambda/\mathrm{tr}(\symbfit M_\lambda)$ and $\symbfit B/\mathrm{tr}(\symbfit B)$, respectively.
This normalization makes $\mu$ control the relative weight of graph regularization independently of the scales of $\symbfit M_\lambda$ and $\symbfit B$.
The penalty-graph Laplacian was set to $\symbfit C=\symbfit I_{n_{\mathrm{a}}}$.

\paragraph{Results}
Fig.~\ref{fig:num_2} shows the classification accuracy as a function of $K$ for KTI and the proposed extensions (GL, TSL, and Center).
Local and the linear integration methods (MPP, GEP, ODC, and LTI) were omitted here, since their trends across $K$ were already shown in Fig.~\ref{fig:num-existing}; Central was omitted for the same reason as in RQ1.

\begin{figure}[t]
  \centering
  \begin{subfigure}{0.49\linewidth}
    \centering
    \includegraphics[width=\linewidth]{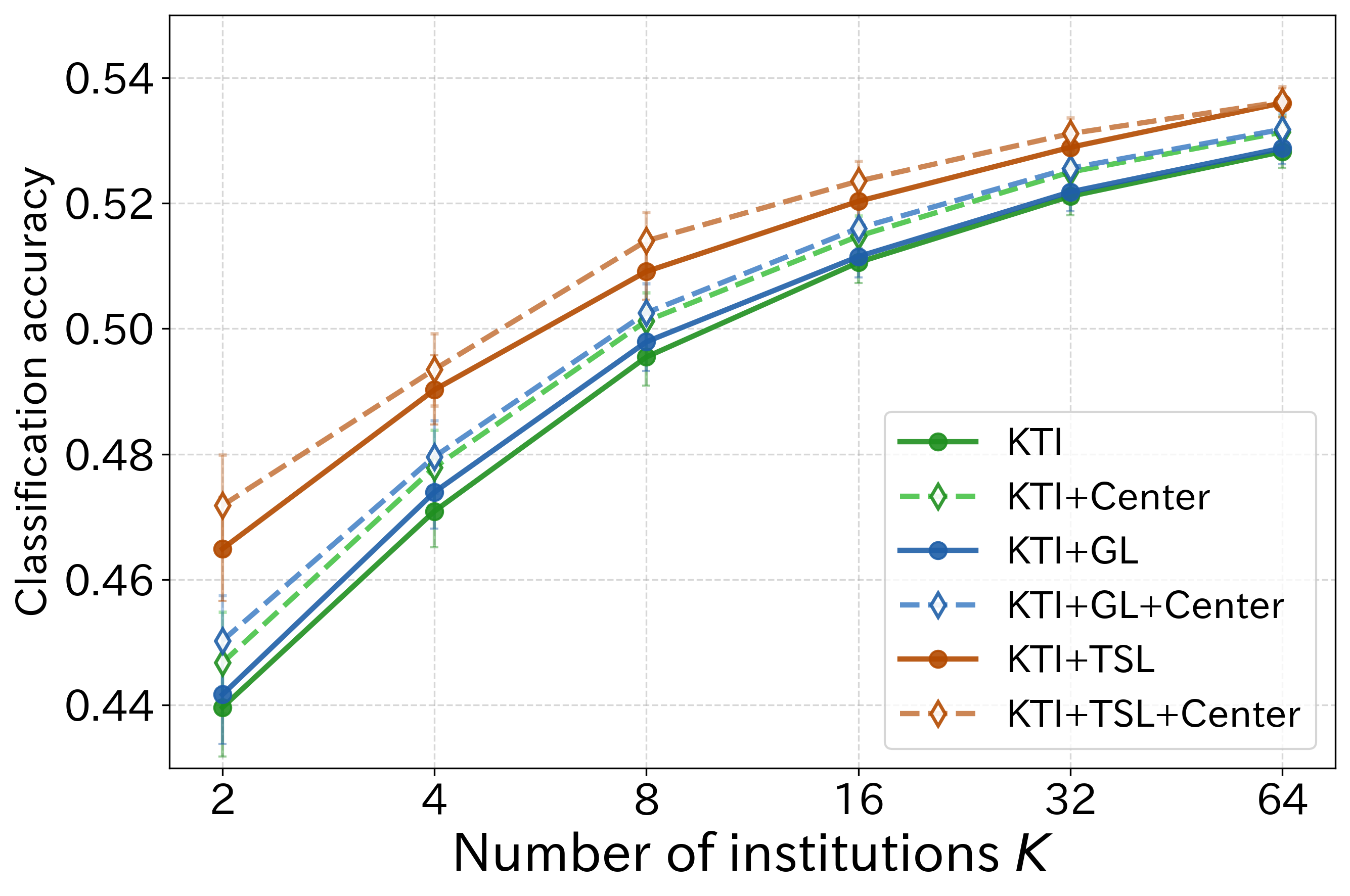}
    \caption{MNIST}
  \end{subfigure}
  \hfill
  \begin{subfigure}{0.49\linewidth}
    \centering
    \includegraphics[width=\linewidth]{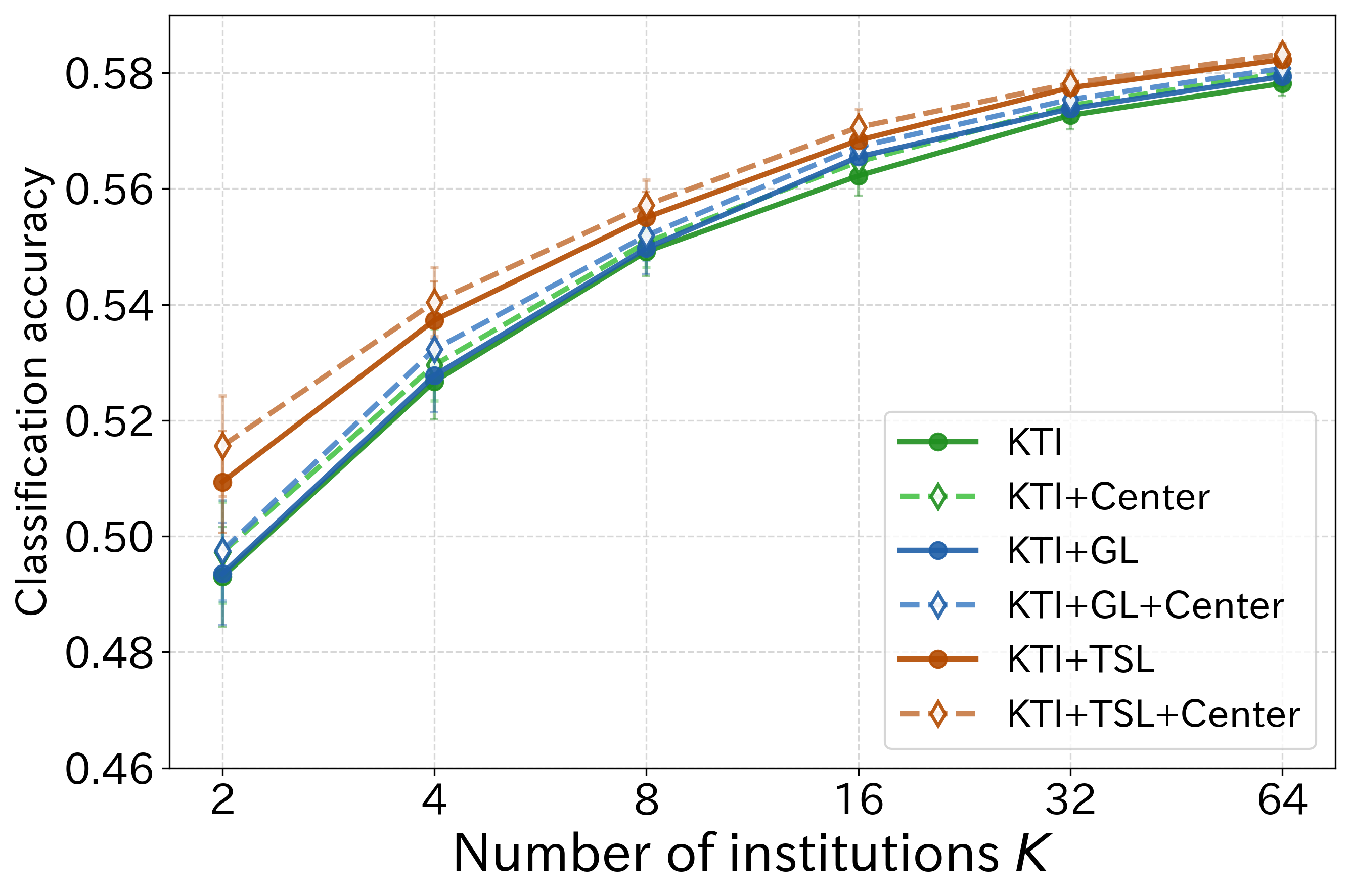}
    \caption{Fashion-MNIST}
  \end{subfigure}
  \caption{Classification accuracy for the proposed extensions in RQ2.}
  \label{fig:num_2}
\end{figure}

On both MNIST and Fashion-MNIST, KTI+TSL and KTI+TSL+Center consistently achieved higher accuracy than KTI without extensions.
Overall, KTI+TSL+Center tended to rank highest, followed by KTI+TSL.
These results suggested that, whereas GL uses only geometry-based neighborhood relationships without considering target-variable information, TSL additionally draws together anchor pairs sharing the same target variable, enabling the collaboration representation to more directly reflect class structure.
Furthermore, combining Center removes constant-direction components from the collaboration representations, which may be one reason why the class structure captured by TSL is reflected more stably.
These results suggested that, in settings where classification performance is prioritized, combining TSL with the centering constraint is an effective strategy.

\subsection{Effect of Anchor Dataset Design}
\label{subsec:rq3_anchor_design}

This experiment examined the effect of anchor dataset design on classification accuracy and computational cost.

\paragraph{Experimental setup}
The number of trials per condition was $N_{\mathrm{seed}}=30$.
We examined two aspects of anchor dataset design: anchor quantity and anchor quality.
The anchor quantity experiment varied $n_{\mathrm a}\in\{100,200,400,800,1600\}$ while fixing $n_{\mathrm a}^{\mathrm{smote}}=100$.
The anchor quality experiment varied $n_{\mathrm a}^{\mathrm{smote}}\in\{10,30,100,300,1000\}$ while fixing $n_{\mathrm a}=1000$.
Here, a larger $n_{\mathrm a}^{\mathrm{smote}}$ means that the anchor dataset was generated from a more diverse set of real source instances.
Computation time was evaluated only in the anchor quantity experiment, because $n_{\mathrm a}$ directly determines the dominant computational cost.
We evaluated two types of computation time: integration-function construction time, which is the time required to construct $\{g_k\}_{k=1}^K$ for all parties, and collaboration-representation construction time, which is the total time required to transform $N$ test instances into collaboration representations.
The latter is a total transformation time over the test dataset and therefore differs from the per-sample complexity in Table~\ref{tab:complexity} by a constant factor under the fixed RQ3 setting.
Computation times were reported as the mean over independent trials for each condition.

\begin{figure}[t]
  \centering
  \begin{subfigure}{0.49\linewidth}
    \centering
    \includegraphics[width=\linewidth]{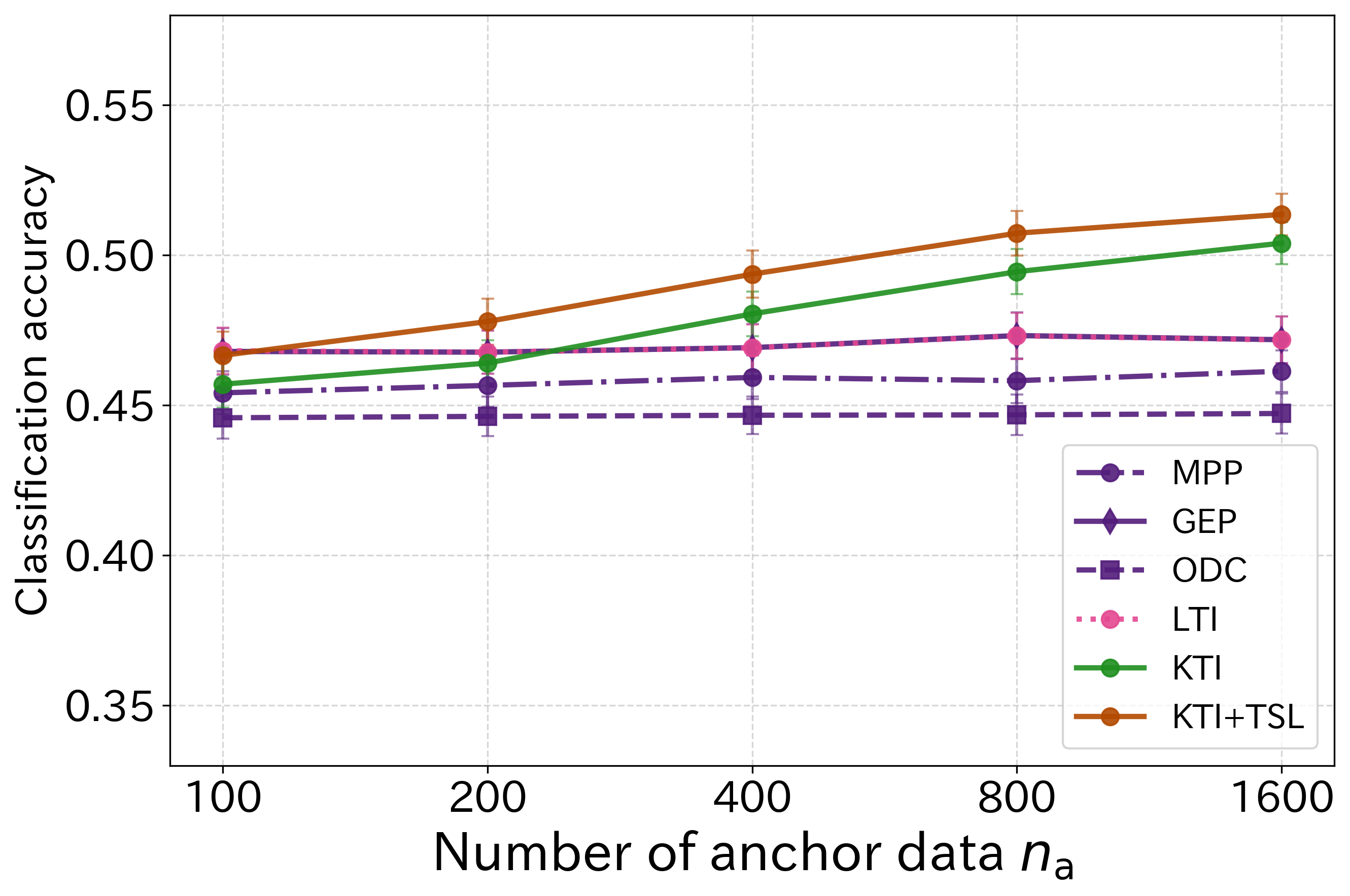}
    \caption{MNIST}
  \end{subfigure}
  \hfill
  \begin{subfigure}{0.49\linewidth}
    \centering
    \includegraphics[width=\linewidth]{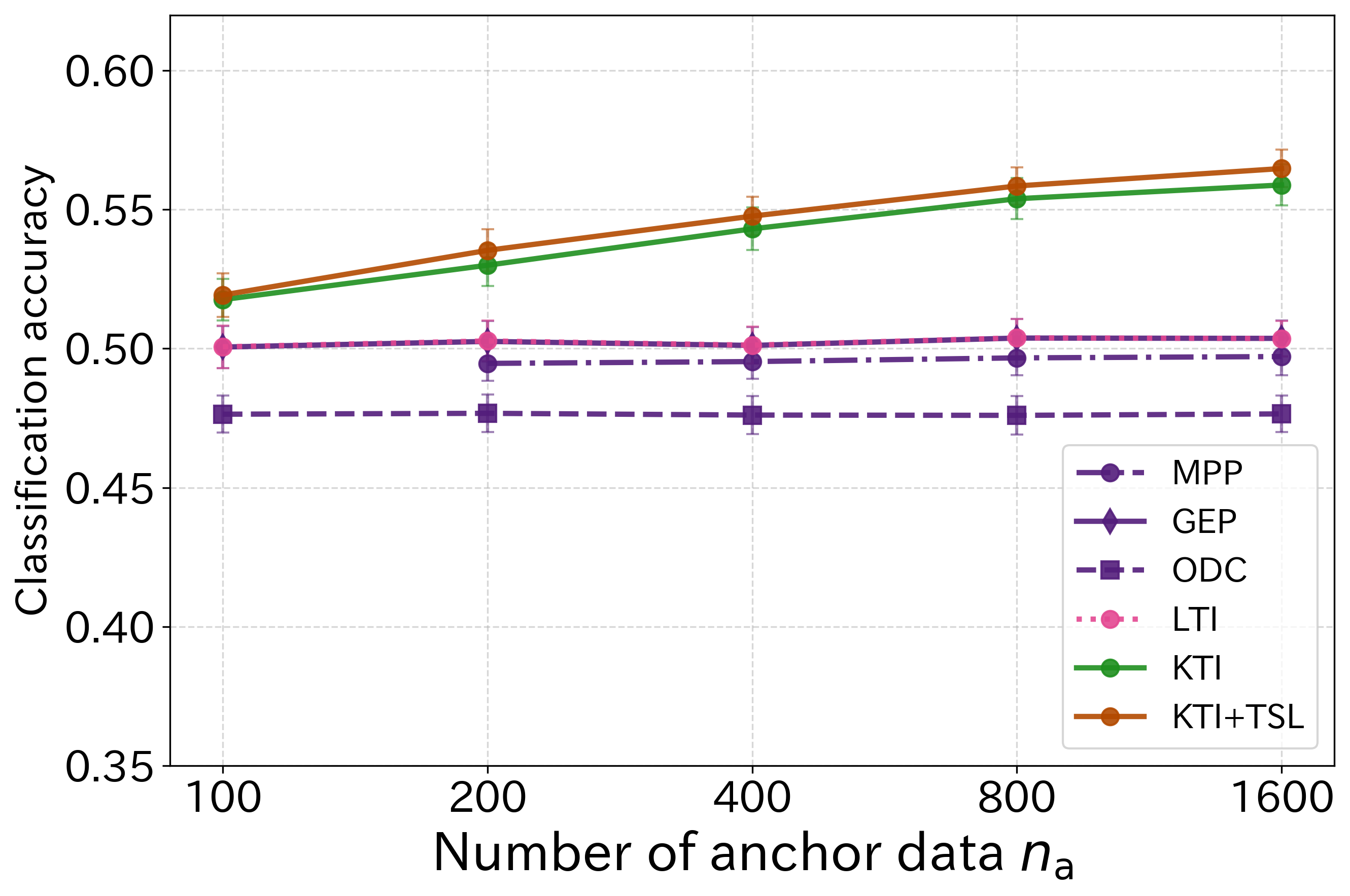}
    \caption{Fashion-MNIST}
  \end{subfigure}
  \caption{Classification accuracy in the anchor-size experiment.}
  \label{fig:rq3_size}
\end{figure}

\begin{figure}[t]
  \centering
  \begin{subfigure}{0.49\linewidth}
    \centering
    \includegraphics[width=\linewidth]{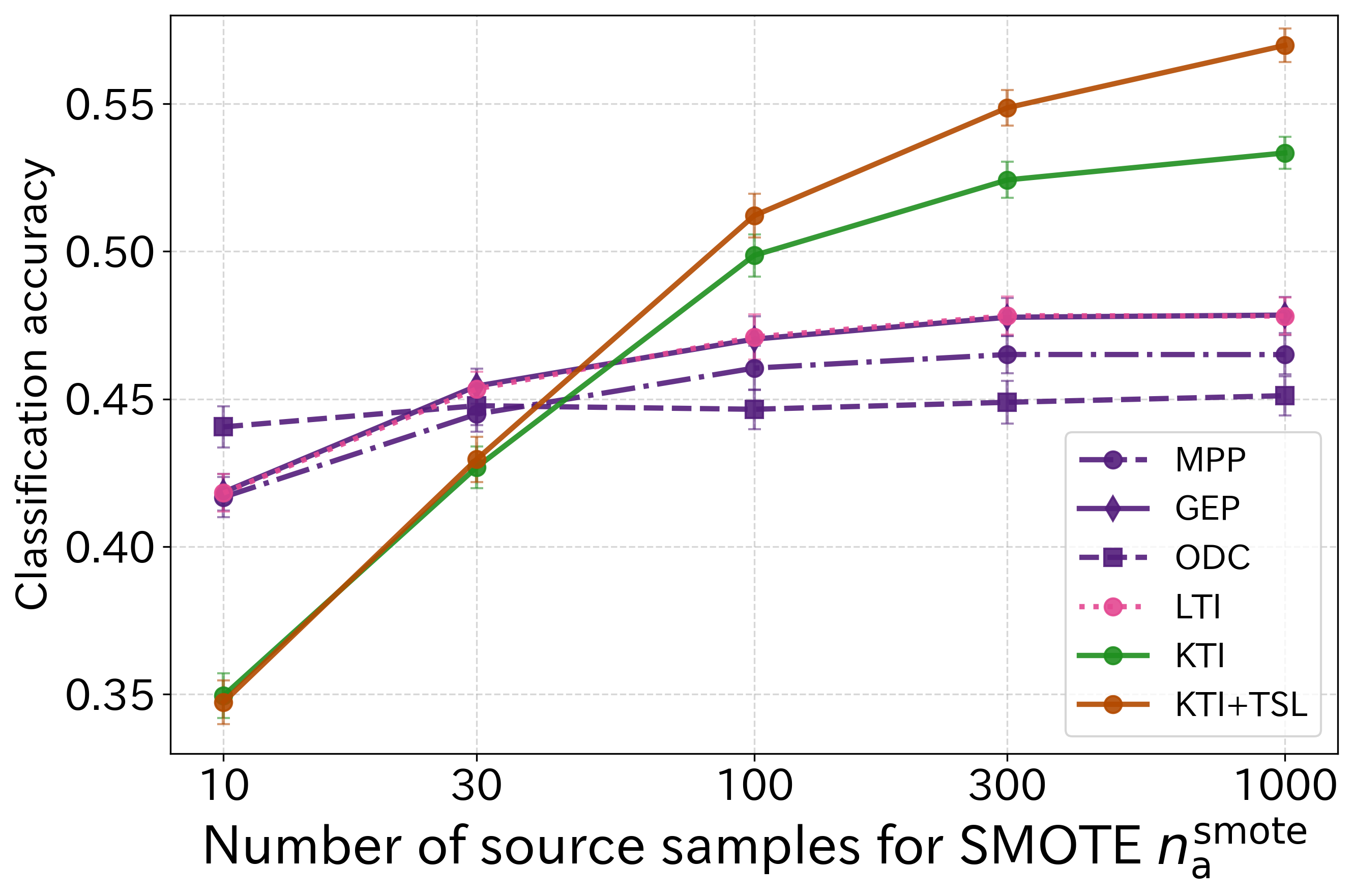}
    \caption{MNIST}
  \end{subfigure}
  \hfill
  \begin{subfigure}{0.49\linewidth}
    \centering
    \includegraphics[width=\linewidth]{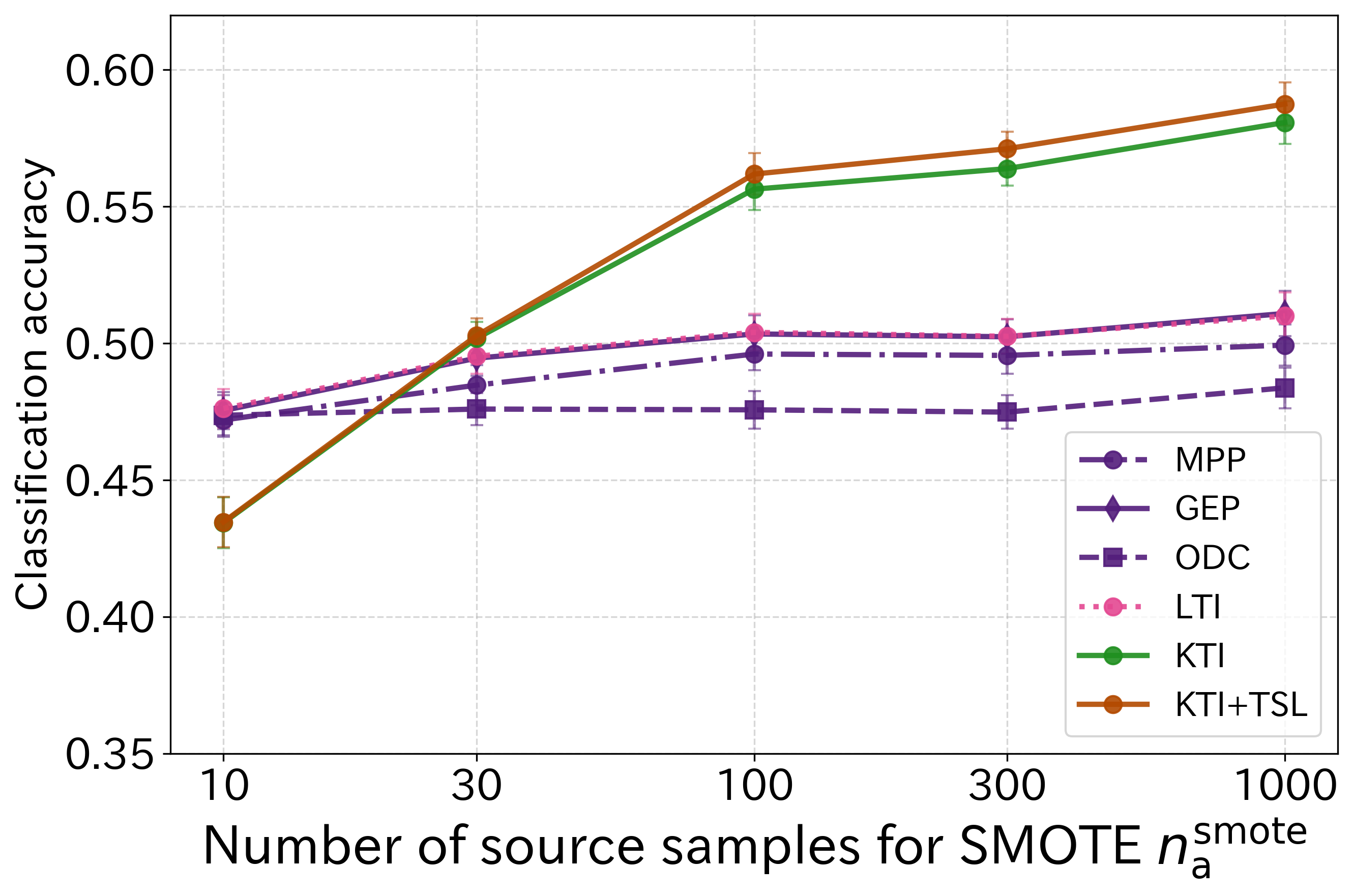}
    \caption{Fashion-MNIST}
  \end{subfigure}
  \caption{Classification accuracy in the anchor-quality experiment.}
  \label{fig:rq3_quality}
\end{figure}

\begin{figure}[t]
  \centering
  \begin{subfigure}{0.49\linewidth}
    \centering
    \includegraphics[width=\linewidth]{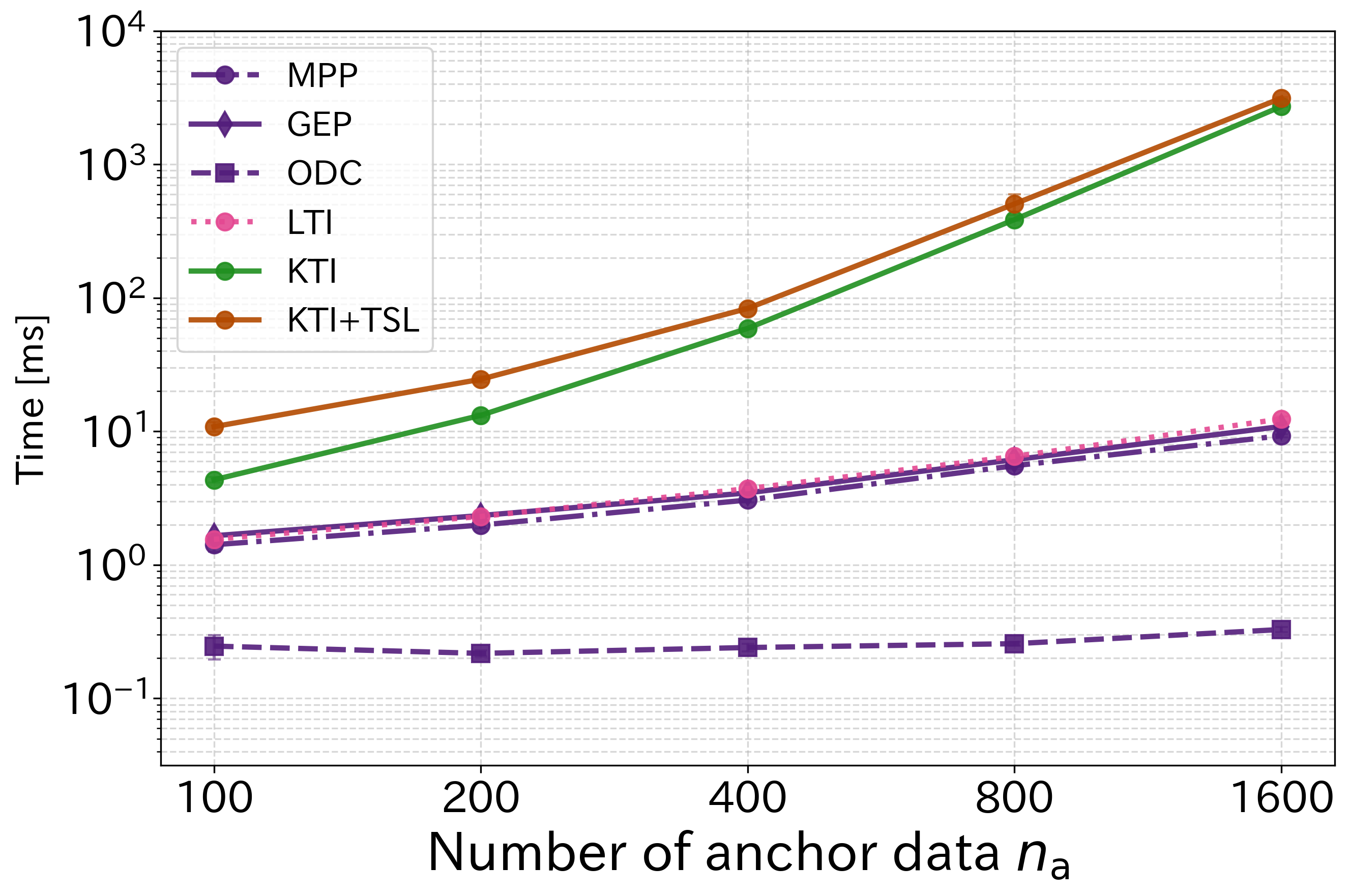}
    \caption{MNIST}
  \end{subfigure}
  \hfill
  \begin{subfigure}{0.49\linewidth}
    \centering
    \includegraphics[width=\linewidth]{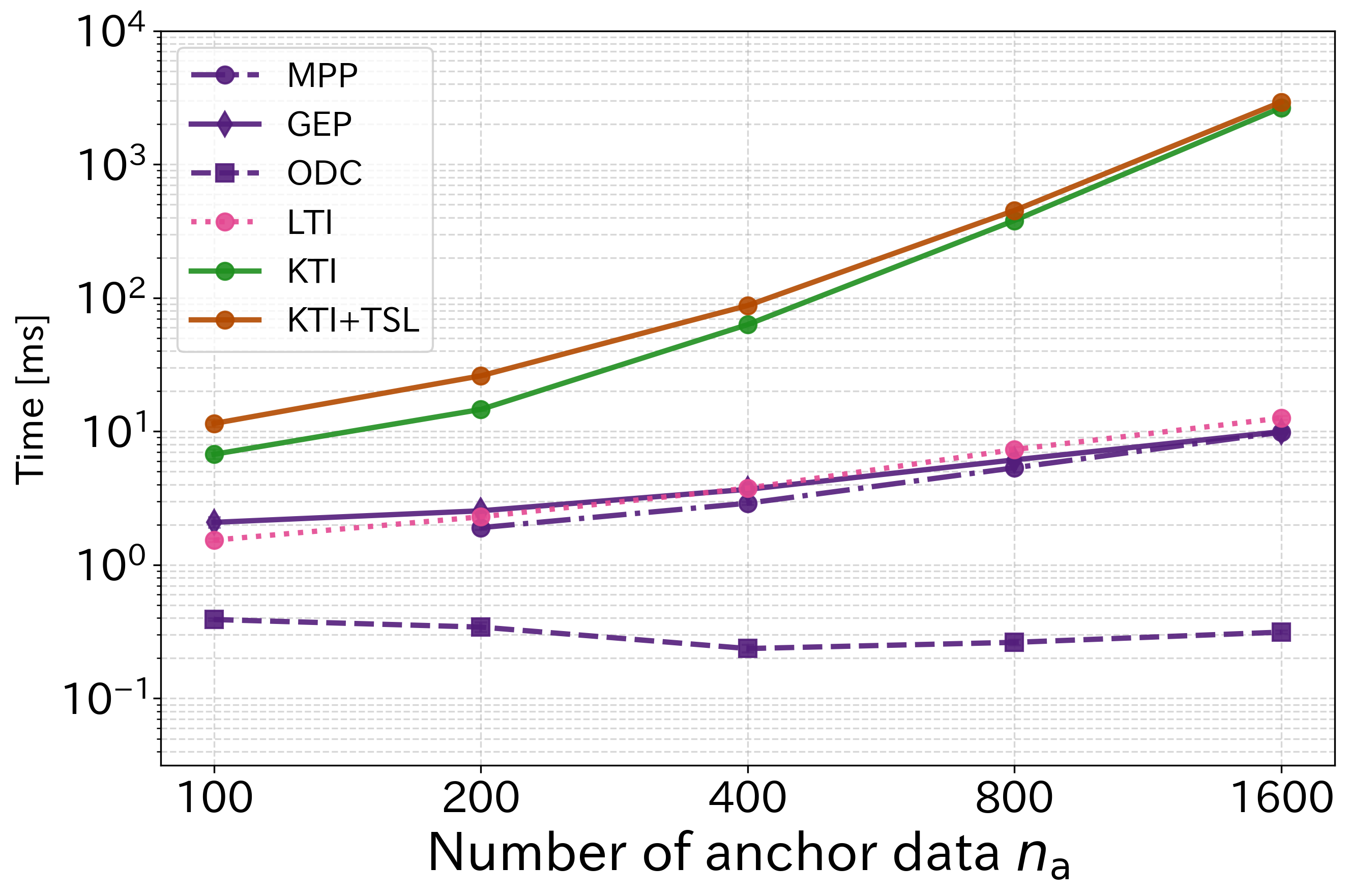}
    \caption{Fashion-MNIST}
  \end{subfigure}
  \caption{Integration-function construction time as a function of anchor size.}
  \label{fig:rq3_func_time}
\end{figure}

\begin{figure}[t]
  \centering
  \begin{subfigure}{0.49\linewidth}
    \centering
    \includegraphics[width=\linewidth]{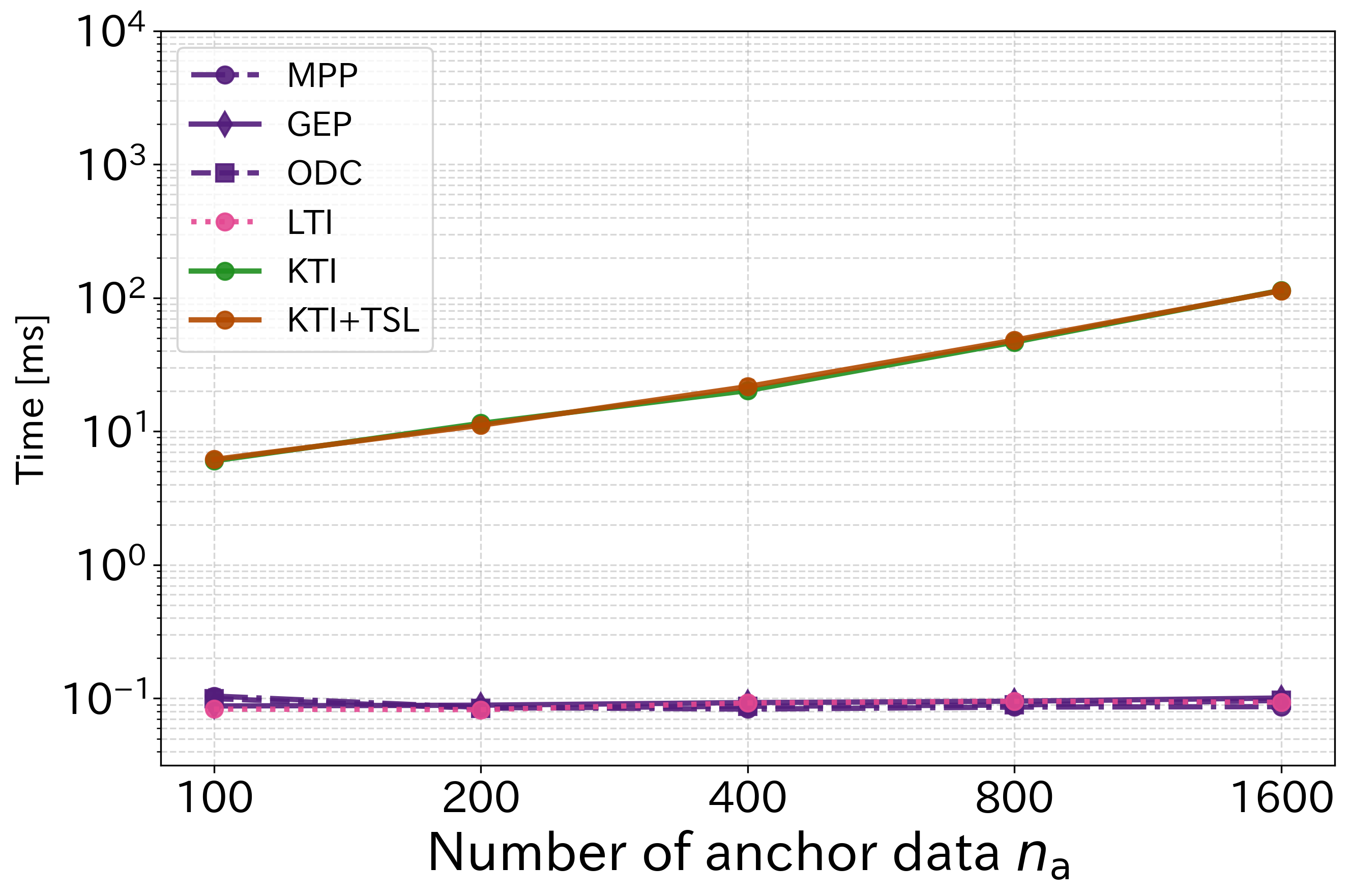}
    \caption{MNIST}
  \end{subfigure}
  \hfill
  \begin{subfigure}{0.49\linewidth}
    \centering
    \includegraphics[width=\linewidth]{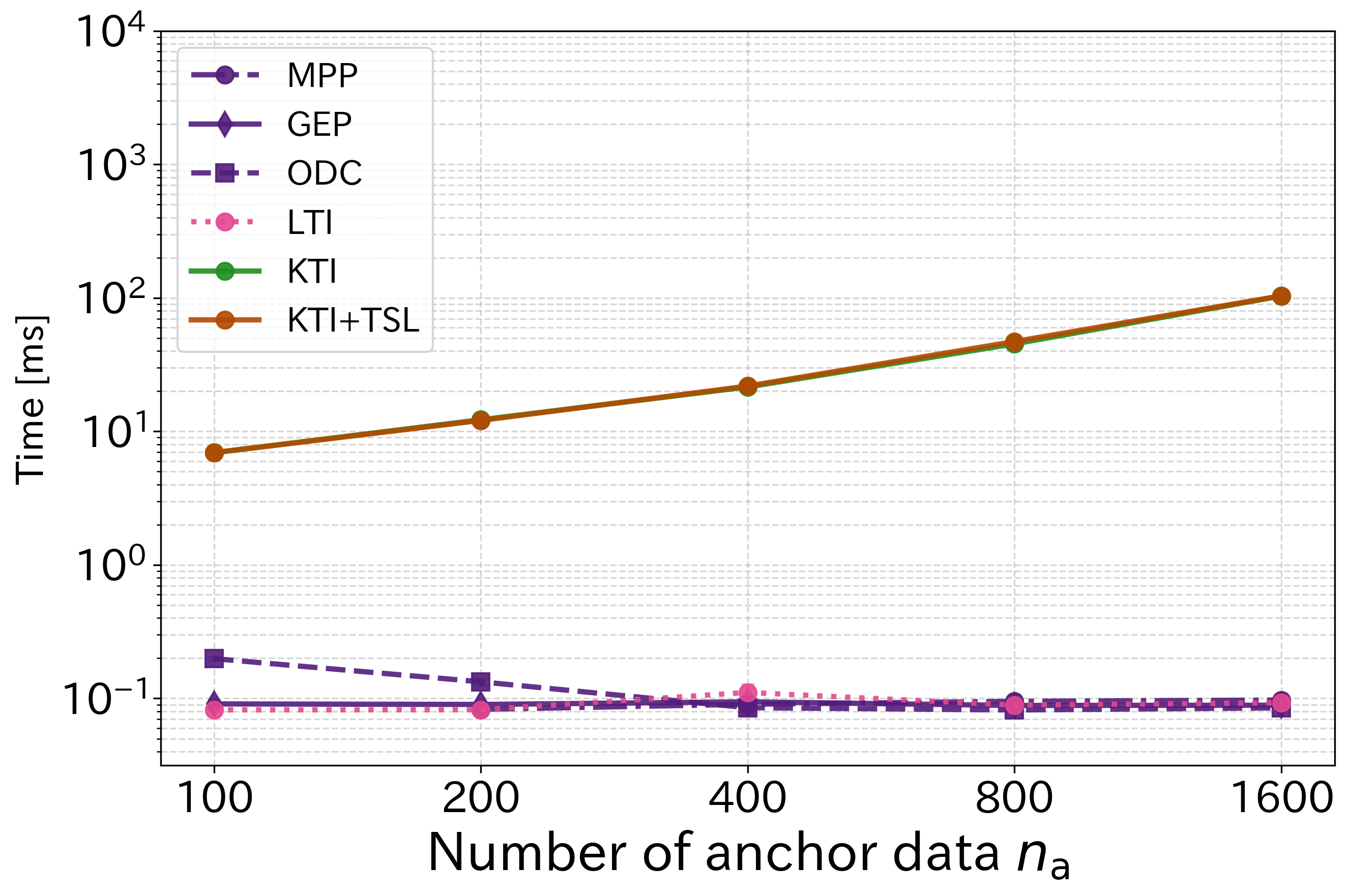}
    \caption{Fashion-MNIST}
  \end{subfigure}
  \caption{Collaboration-representation construction time as a function of anchor size.}
  \label{fig:rq3_repr_time}
\end{figure}

\paragraph{Classification accuracy}
Classification accuracy for the anchor-size and anchor-quality experiments is shown in Figs.~\ref{fig:rq3_size} and~\ref{fig:rq3_quality}, respectively.
First, regarding the effect of anchor quantity, classification accuracy improved overall as $n_{\mathrm a}$ increased.
The improvement was particularly large for KTI-based methods (KTI, KTI+TSL), indicating that adding more anchor data was especially beneficial for these methods under these settings.

Next, regarding the effect of anchor quality, the accuracy improvement from increasing $n_{\mathrm a}^{\mathrm{smote}}$ was more pronounced than that from increasing $n_{\mathrm a}$.
This tendency was observed on both MNIST and Fashion-MNIST and was especially strong for KTI-based methods, suggesting that these methods were more sensitive to the diversity and representativeness of the SMOTE source data than to the sheer quantity of anchor data under these settings.

\paragraph{Computation time}
Figs.~\ref{fig:rq3_func_time} and~\ref{fig:rq3_repr_time} show the integration-function construction time and collaboration-representation construction time as functions of $n_{\mathrm a}$, respectively, displayed on log-log scales.
Note that because $n_{\mathrm a}$ is fixed in the anchor quality experiment, computation time evaluation applies to the anchor quantity experiment only.

According to Table~\ref{tab:complexity}, the theoretical complexity for linear integration methods (GEP, MPP, ODC, LTI) is approximately $O(n_{\mathrm a})$ for function construction and $O(1)$ for representation construction.
KTI-based methods (KTI, KTI+TSL) have theoretical complexity $O(n_{\mathrm a}^3)$ and $O(n_{\mathrm a})$, respectively.

These results indicated that increasing $n_{\mathrm a}$ improved classification accuracy but substantially increased computation time, particularly for KTI-based methods.
Moreover, even with the total anchor size fixed, increasing $n_{\mathrm a}^{\mathrm{smote}}$ yielded more pronounced accuracy improvements.
Therefore, in practice, it is important not only to increase the number of anchor instances but also to ensure sufficient diversity and representativeness of the source data used for anchor generation.

\subsection{Utility--Privacy Trade-off}
\label{sec:rq4}

This experiment evaluated the utility--privacy trade-off, using classification accuracy as the utility measure and reconstruction accuracy, defined below, as the privacy-risk measure.

\paragraph{Experimental setup}
We varied the dimensionality reduction method over PCA, Kernel PCA, and UMAP, and set the intermediate dimension to $\tilde{d}(k)\in\{4,16,64\}$.
The datasets were MNIST and Fashion-MNIST, and the number of trials per condition was $N_{\mathrm{seed}}=30$.
In this experiment, we used the default training split $K=10$, $N=1000$, and $n(k)=100$.
To evaluate reconstruction risk under a conservative anchor-leakage scenario, we used real-data anchors instead of SMOTE-generated anchors.
The anchor dataset consisted of 100 real-data instances per label, for a total of $n_{\mathrm a}=1000$ instances; no SMOTE augmentation was applied.
The RBF kernel bandwidth for Kernel PCA was determined by the median heuristic applied to standardized data.

\paragraph{Reconstruction attack}
We considered an analyst-side reconstruction attack caused by leakage of part of the anchor dataset held by the parties.
In the standard DC analysis setting, the analyst receives intermediate representations but not the corresponding original feature vectors.
In this experiment, we assumed an additional leakage scenario in which the original feature vectors of a subset of anchor instances held by the parties were leaked to the analyst.
Such leakage can occur, for example, when one party colludes with the analyst; in the worst case, the analyst may obtain the entire anchor dataset.
In our evaluation, we considered a partial-leakage setting in which only anchor instances from selected labels were leaked.

For a target party $k$, let $\mathcal{I}\subset[n_{\mathrm{a}}]$ denote the indices of the leaked anchor instances.
The attacker therefore had access to the leaked pairs $(\symbfit{A}_{\mathcal I,:},\tilde{\symbfit{A}}^{(k)}_{\mathcal I,:})$ and used them to learn a reconstruction function from intermediate representations to the original feature space.
For MNIST, the leaked anchor pairs were 300 instances corresponding to digits $0$, $1$, and $2$ (100 per label), and reconstruction risk was evaluated for the remaining digits $3$--$9$.
For Fashion-MNIST, the leaked anchor pairs were 300 instances corresponding to labels $0$ (T-shirt/top), $1$ (Trouser), and $2$ (Pullover), and reconstruction risk was evaluated for the remaining labels $3$--$9$.

We evaluated three reconstruction attacks: LR, MLP, and PINV.
All attacks used only the leaked anchor pairs $(\symbfit{A}_{\mathcal I,:},\tilde{\symbfit{A}}^{(k)}_{\mathcal I,:})$ for training.
Let $\bar{\symbfit{A}}_{\mathcal I}$ and $\bar{\tilde{\symbfit{A}}}_{\mathcal I}^{(k)}$ denote the mean row vectors of $\symbfit{A}_{\mathcal I,:}$ and $\tilde{\symbfit{A}}^{(k)}_{\mathcal I,:}$, respectively, whose entries are the column-wise means of the corresponding matrices.
The centered matrices are defined as
\[
\symbfit{A}_{\mathcal I}^{\mathrm c}
=
\symbfit{A}_{\mathcal I,:}-\bar{\symbfit{A}}_{\mathcal I},
\qquad
\tilde{\symbfit{A}}_{\mathcal I}^{(k),\mathrm c}
=
\tilde{\symbfit{A}}^{(k)}_{\mathcal I,:}
-
\bar{\tilde{\symbfit{A}}}_{\mathcal I}^{(k)}
\]
LR was a linear regression attack that learned a centered linear map from $\tilde{\symbfit{A}}_{\mathcal I}^{(k),\mathrm c}$ to $\symbfit{A}_{\mathcal I}^{\mathrm c}$.
MLP learned the same reconstruction direction using a nonlinear multilayer perceptron.
For MLP, we used one hidden layer with 128 units and ReLU activations, optimized with Adam for at most 600 iterations using early stopping with a 20\% validation split.
For both LR and MLP, the reconstruction direction was
\[
\tilde{\symbfit{A}}_{\mathcal I}^{(k),\mathrm c}
\mapsto
\symbfit{A}_{\mathcal I}^{\mathrm c},
\]
and $\bar{\symbfit{A}}_{\mathcal I}$ was added back after prediction.
In contrast, PINV first estimated a forward linear map
\[
\symbfit{A}_{\mathcal I}^{\mathrm c}\symbfit{F}_{\mathrm{PINV}}^{(k)}
\simeq
\tilde{\symbfit{A}}_{\mathcal I}^{(k),\mathrm c},
\]
and reconstructed original data using $(\symbfit{F}_{\mathrm{PINV}}^{(k)})^\dagger$ with mean correction.
The learned reconstruction functions were then applied to the intermediate representations $\tilde{\symbfit X}^{(k)}$ of evaluation instances whose labels were not included in the leaked anchor pairs, yielding reconstructed data in the original feature space.
\textit{Reconstruction accuracy} was defined as the label-match rate when reconstructed images were passed to a \texttt{RandomForestClassifier} trained independently on 5000 original instances with default settings.
A higher reconstruction accuracy indicates higher reconstruction risk and greater privacy leakage.
For each condition, reconstruction accuracy was computed over 50 reconstructed instances per evaluation label (350 instances in total).
We reported the maximum reconstruction accuracy among LR, MLP, and PINV, together with the corresponding attack method.

\paragraph{Classification and reconstruction results}
Table~\ref{tab:rq4-tradeoff-main-wide} reports the classification accuracy, reconstruction accuracy, and strongest attack method for KTI+TSL under different dimensionality reduction methods and intermediate dimensions.
Classification accuracy represents downstream utility, whereas reconstruction accuracy represents reconstruction risk.

\begin{table*}[t]
\centering
\tiny
\setlength{\tabcolsep}{2.4pt}
\renewcommand{\arraystretch}{1.10}
\caption{Classification and reconstruction results of KTI+TSL across dimensionality reduction methods and intermediate dimensions.}
\label{tab:rq4-tradeoff-main-wide}
\begin{tabular*}{\textwidth}{@{\extracolsep{\fill}}llccccccccc@{}}
\toprule
& &
\multicolumn{3}{c}{$\tilde d(k)=4$} &
\multicolumn{3}{c}{$\tilde d(k)=16$} &
\multicolumn{3}{c}{$\tilde d(k)=64$} \\
\cmidrule(lr){3-5}\cmidrule(lr){6-8}\cmidrule(lr){9-11}
Dataset & DimRed
& Acc. ($\uparrow$) & Recon. ($\downarrow$) & Atk.
& Acc. ($\uparrow$) & Recon. ($\downarrow$) & Atk.
& Acc. ($\uparrow$) & Recon. ($\downarrow$) & Atk. \\
\midrule
\multirow[t]{3}{*}{MNIST}
& PCA
& \valci{.394}{.008} & \valci{.267}{.010} & PINV
& \valci{.824}{.004} & \valci{.632}{.007} & PINV
& \valci{.859}{.004} & \valci{.834}{.005} & PINV \\
& K-PCA
& \valci{.329}{.005} & \valci{.188}{.007} & LR
& \valci{.685}{.006} & \valci{.540}{.014} & LR
& \valci{.753}{.006} & \valci{.684}{.010} & LR \\
& UMAP
& \valci{.457}{.008} & \valci{.143}{.000} & PINV
& \valci{.594}{.005} & \valci{.159}{.017} & LR
& \valci{.598}{.005} & \valci{.153}{.021} & LR \\
\midrule
\multirow[t]{3}{*}{Fashion-MNIST}
& PCA
& \valci{.573}{.004} & \valci{.211}{.007} & PINV
& \valci{.709}{.005} & \valci{.557}{.011} & PINV
& \valci{.720}{.006} & \valci{.687}{.015} & PINV \\
& K-PCA
& \valci{.489}{.007} & \valci{.142}{.007} & LR
& \valci{.684}{.006} & \valci{.326}{.011} & LR
& \valci{.717}{.007} & \valci{.454}{.013} & LR \\
& UMAP
& \valci{.527}{.007} & \valci{.151}{.029} & MLP
& \valci{.607}{.006} & \valci{.220}{.033} & LR
& \valci{.617}{.007} & \valci{.189}{.021} & LR \\
\bottomrule
\end{tabular*}

\vspace{2pt}
\footnotesize
Acc. denotes classification accuracy, Recon. denotes reconstruction accuracy, and Atk. denotes the attack method with the highest reconstruction accuracy.
Arrows indicate the preferred direction.
K-PCA denotes Kernel PCA.
\end{table*}

Each numerical cell reports the mean and the 95\% confidence interval.

As shown in Table~\ref{tab:rq4-tradeoff-main-wide}, classification accuracy tended to increase as $\tilde{d}(k)$ increased.
However, reconstruction risk also increased substantially for PCA and Kernel PCA.
In contrast, UMAP kept reconstruction accuracy relatively low even when $\tilde{d}(k)$ increased, indicating stronger confidentiality against reconstruction attacks.
Its classification accuracy was often lower than that of PCA, indicating that the lower reconstruction risk came with reduced utility in several settings.
These results suggest a utility--privacy trade-off: increasing the expressive power of the intermediate representation can improve classification performance, but it can also increase vulnerability to reconstruction attacks.
Regarding the reconstruction attack methods, PINV tended to be most effective under PCA, whereas LR was most effective in many conditions under Kernel PCA and UMAP.
This tendency suggested that, under PCA, the obfuscation function $f_k$ of party $k$ can be approximated by a linear mapping matrix $\symbfit{F}^{(k)}$ as $\tilde{\symbfit{X}}^{(k)}\simeq \symbfit{X}^{(k)} \symbfit{F}^{(k)}$, making reconstruction via the pseudoinverse $(\symbfit{F}^{(k)})^\dagger$ effective.
In contrast, since the mappings in Kernel PCA and UMAP are nonlinear, pseudoinverse reconstruction based on a single linear mapping was less effective.

\paragraph{Integration-method comparison and qualitative reconstruction}
To avoid conflating the utility--privacy analysis with integration-method comparison, Table~\ref{tab:rq4-tradeoff-main-wide} fixes the integration method to KTI+TSL.
Table~\ref{tab:rq4-integration-accuracy-wide} separately reports the classification accuracy of LTI, KTI, and KTI+TSL under the same dimensionality reduction settings.

\begin{table*}[t]
\centering
\tiny
\setlength{\tabcolsep}{2.2pt}
\renewcommand{\arraystretch}{1.08}
\caption{Classification accuracy of integration methods under different dimensionality reduction methods.}
\label{tab:rq4-integration-accuracy-wide}
\begin{tabular*}{\textwidth}{@{\extracolsep{\fill}}llccccccccc@{}}
\toprule
& &
\multicolumn{3}{c}{$\tilde d(k)=4$} &
\multicolumn{3}{c}{$\tilde d(k)=16$} &
\multicolumn{3}{c}{$\tilde d(k)=64$} \\
\cmidrule(lr){3-5}\cmidrule(lr){6-8}\cmidrule(lr){9-11}
Dataset & DimRed
& LTI & KTI & KTI+TSL
& LTI & KTI & KTI+TSL
& LTI & KTI & KTI+TSL \\
\midrule
\multirow[t]{3}{*}{MNIST}
& PCA
& \valci{\textbf{.408}}{.007} & \valci{.336}{.011} & \valci{.394}{.008}
& \valci{.755}{.007} & \valci{.517}{.012} & \valci{\textbf{.824}}{.004}
& \valci{.751}{.013} & \valci{.673}{.008} & \valci{\textbf{.859}}{.004} \\
& K-PCA
& \valci{\textbf{.346}}{.006} & \valci{.320}{.006} & \valci{.329}{.005}
& \valci{.636}{.007} & \valci{.673}{.006} & \valci{\textbf{.685}}{.006}
& \valci{.635}{.014} & \valci{.717}{.009} & \valci{\textbf{.753}}{.006} \\
& UMAP
& \valci{.349}{.006} & \valci{.425}{.008} & \valci{\textbf{.457}}{.008}
& \valci{.512}{.007} & \valci{.567}{.005} & \valci{\textbf{.594}}{.005}
& \valci{.554}{.007} & \valci{.586}{.005} & \valci{\textbf{.598}}{.005} \\
\midrule
\multirow[t]{3}{*}{Fashion-MNIST}
& PCA
& \valci{.538}{.008} & \valci{.502}{.005} & \valci{\textbf{.573}}{.004}
& \valci{\textbf{.721}}{.006} & \valci{.631}{.007} & \valci{.709}{.005}
& \valci{\textbf{.725}}{.010} & \valci{.672}{.007} & \valci{.720}{.006} \\
& K-PCA
& \valci{.476}{.008} & \valci{.484}{.007} & \valci{\textbf{.489}}{.007}
& \valci{.664}{.007} & \valci{.681}{.007} & \valci{\textbf{.684}}{.006}
& \valci{.656}{.012} & \valci{.709}{.008} & \valci{\textbf{.717}}{.007} \\
& UMAP
& \valci{.416}{.005} & \valci{.509}{.008} & \valci{\textbf{.527}}{.007}
& \valci{.544}{.008} & \valci{.598}{.007} & \valci{\textbf{.607}}{.006}
& \valci{.575}{.007} & \valci{.609}{.007} & \valci{\textbf{.617}}{.007} \\
\bottomrule
\end{tabular*}

\vspace{2pt}
\footnotesize
Bold values indicate the highest classification accuracy among the three integration methods for each dataset, dimensionality reduction method, and intermediate dimension.
\end{table*}

Table~\ref{tab:rq4-integration-accuracy-wide} further shows that LTI was competitive under PCA, whereas KTI and KTI+TSL were generally more effective under Kernel PCA and UMAP.
This result supported the main claim that nonlinear integration is useful when intermediate representations are obtained through nonlinear dimensionality reduction.
Adding TSL regularization improved accuracy in most conditions, indicating that target-variable-aware graph regularization could also be beneficial under these settings.

The qualitative reconstruction results at $\tilde{d}(k)=64$ are shown in Fig.~\ref{fig:recon-64}; the corresponding quantitative values are reported in the $\tilde{d}(k)=64$ columns of Table~\ref{tab:rq4-tradeoff-main-wide}.
The first row for each dataset shows original images, and the subsequent rows show the best reconstruction results for each dimensionality reduction method; the attack method achieving the highest reconstruction accuracy is shown in parentheses.
For Fashion-MNIST, reconstructed images tended to preserve class-level patterns rather than instance-specific details of the original input.
This may be because UMAP tends to construct an intermediate representation space that separates clusters, making it easier for the inverse mapping to pull an input point toward typical instances in the nearest learned cluster.

\begin{figure}[t]
  \centering
  \includegraphics[width=0.86\linewidth]{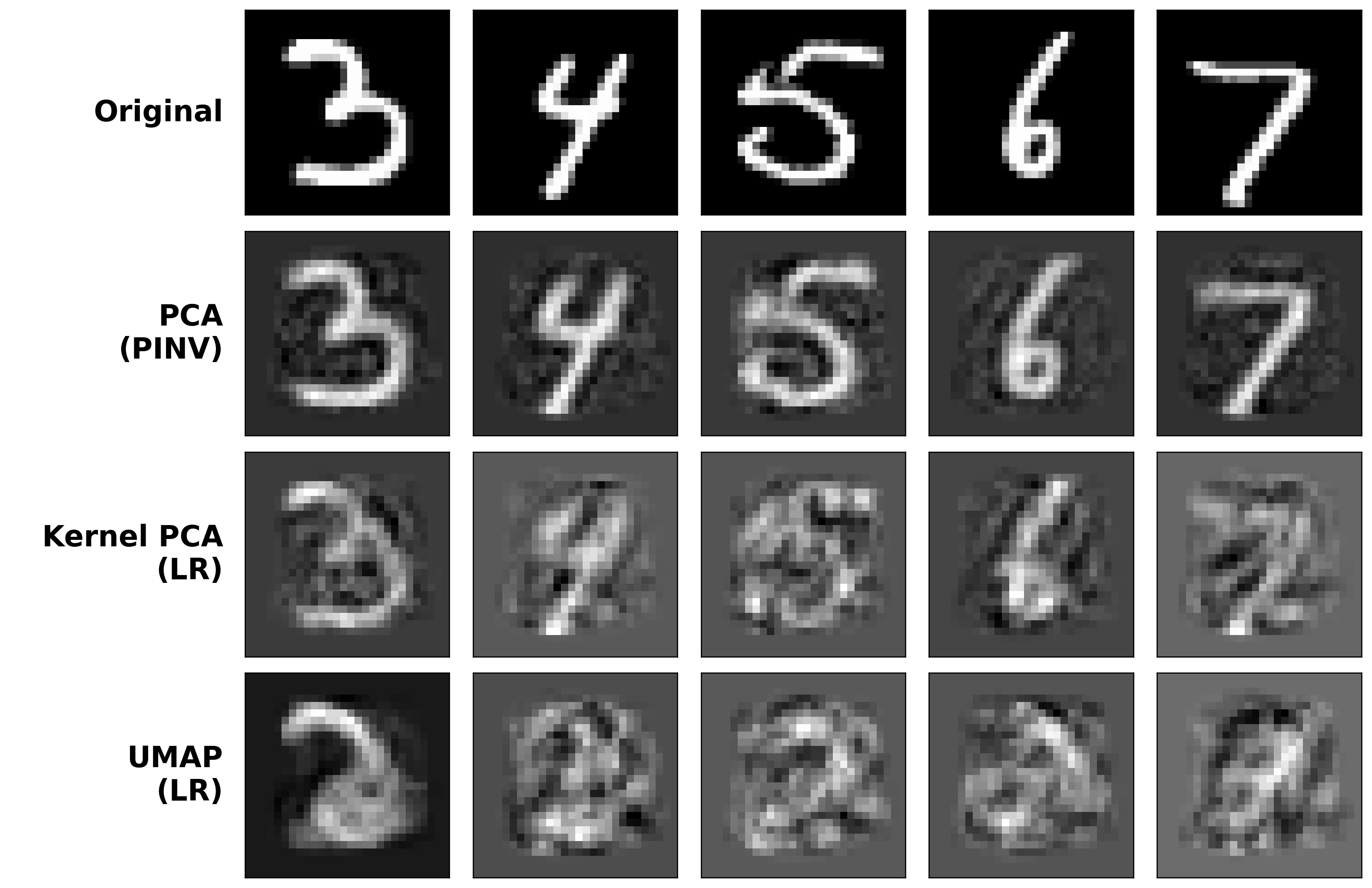}\\[0.4em]
  \includegraphics[width=0.86\linewidth]{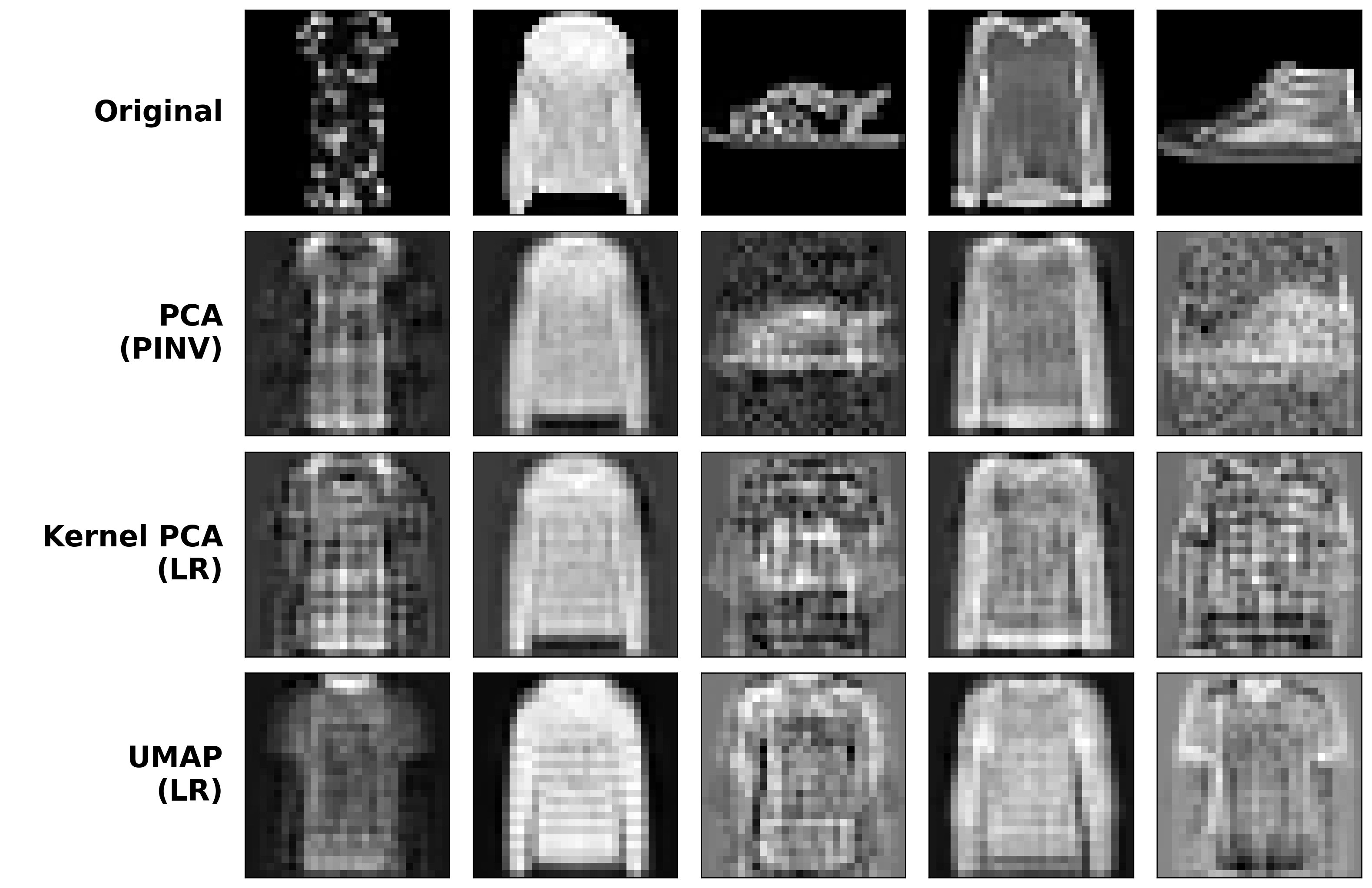}
  \caption{Qualitative reconstruction results for MNIST and Fashion-MNIST at $\tilde{d}(k)=64$.}
  \label{fig:recon-64}
\end{figure}

Overall, these results showed that dimensionality reduction methods and intermediate dimensions should be selected by jointly considering classification accuracy and reconstruction risk.

\FloatBarrier

\section{Conclusion}
\label{sec:conclusion}
In this study, we addressed the integration problem in DC analysis by first formulating a linear integration problem, characterizing its globally optimal solution, and then extending it to a nonlinear setting through kernelization.
This nonlinear extension enables the integration of intermediate representations obtained through nonlinear dimensionality reduction.
In the proposed method, the estimation of each integration function for a fixed target representation is reduced to a finite-dimensional kernel ridge regression-type problem, and the estimation of the target representation is reduced to an eigenvalue problem.
This formulation makes nonlinear integration tractable through eigendecomposition.
We also introduced graph regularization and a centering constraint to incorporate geometric structure and target-variable information of the anchor dataset into the target representation.

Numerical experiments showed that, under nonlinear dimensionality reduction, the proposed KTI-based methods tended to achieve higher classification accuracy than existing linear integration methods.
The results also showed that combining TSL using target-variable information with the centering constraint improved classification accuracy, and that not only the quantity of anchor data but also the distributional characteristics of the source data used to generate anchors substantially affected classification accuracy.
Furthermore, the choice of dimensionality reduction method and intermediate dimension affected both classification accuracy and reconstruction risk.
In particular, the results suggest that UMAP is useful when confidentiality against reconstruction attacks is prioritized, although it does not necessarily maximize classification accuracy.

These results indicate that, in practice, integration and dimensionality reduction methods should be selected by jointly considering classification accuracy, anchor data design, computational cost, and reconstruction risk.
Future work includes extending the proposed method to large-scale datasets using kernel approximation and constructing a DC framework that can adjust dimensionality reduction methods and intermediate dimensions according to party-specific confidentiality and accuracy requirements.
Additionally, evaluating the privacy guarantees of the proposed method against stronger reconstruction attack models remains an important open problem.

\appendix

\section{Proof of Theorem 2}
\label{app:proof-kernel-reduction}
\begin{proof}
Firstly, we fix $\symbfit Z\in\mathbb R^{n_{\mathrm a}\times\hat d}$.
Then, the inner minimization in Eq.~\eqref{eq:kernel-formulation} can be written as the following multi-output kernel ridge regression-type problem for the fixed $\symbfit Z$:
\begin{equation}
\label{eq:kernel-proof-inner-problem}
\begin{aligned}
\min_{\{g_k\in\mathcal H_k^{\hat d}\}_{k=1}^K}
\quad
&
\sum_{k=1}^K
\left(
\|g_k(\tilde{\symbfit A}^{(k)})-\symbfit Z\|_{\mathrm F}^2
+\lambda\|g_k\|_{\mathcal H_k^{\hat d}}^2
\right) .
\end{aligned}
\end{equation}
Eq.~\eqref{eq:kernel-proof-inner-problem} is separable with respect to $k\in[K]$.
Therefore, each subproblem can be solved independently.
Since $\mathcal H_k^{\hat d}$ is a product RKHS, the representer theorem applies coordinate-wise to each component of $g_k$.
By the representer theorem~\citep{Generalized_Representer_Theorem}, an optimal solution for each $k\in[K]$ can be represented using a coefficient matrix $\symbfit\Gamma_k\in\mathbb R^{n_{\mathrm a}\times\hat d}$ as follows:
\begin{equation}
\label{eq:kernel-proof-representer-form}
g_k^\star(\symbfit x)=\symbfit\kappa_k(\symbfit x)\symbfit\Gamma_k .
\end{equation}
Using this representation, the function values and the RKHS norm are written as follows:
\begin{equation}
\label{eq:kernel-proof-representer-identities}
\begin{aligned}
g_k^\star(\tilde{\symbfit A}^{(k)})=\symbfit \Sigma_k\symbfit\Gamma_k,
\qquad
\|g_k^\star\|_{\mathcal H_k^{\hat d}}^2
=\mathrm{tr}(\symbfit\Gamma_k^\top\symbfit \Sigma_k\symbfit\Gamma_k).
\end{aligned}
\end{equation}
Consequently, by Eqs.~\eqref{eq:kernel-proof-representer-form} and~\eqref{eq:kernel-proof-representer-identities}, Eq.~\eqref{eq:kernel-proof-inner-problem} is decomposed into the following matrix optimization problem for each $k\in[K]$:
\begin{equation}
\label{eq:kernel-proof-gamma-problem}
\min_{\symbfit\Gamma_k\in\mathbb R^{n_{\mathrm a}\times\hat d}}
\|\symbfit \Sigma_k\symbfit\Gamma_k-\symbfit Z\|_{\mathrm F}^2
+\lambda\,\mathrm{tr}(\symbfit\Gamma_k^\top\symbfit \Sigma_k\symbfit\Gamma_k).
\end{equation}
Next, to solve Eq.~\eqref{eq:kernel-proof-gamma-problem}, we compute the first-order condition with respect to $\symbfit\Gamma_k$.
The optimality condition can be written as follows:
\begin{equation}
\label{eq:kernel-proof-gamma-opt-cond}
\begin{aligned}
\nabla_{\symbfit\Gamma_k}
\left\{
\|\symbfit \Sigma_k\symbfit\Gamma_k-\symbfit Z\|_{\mathrm F}^2
+\lambda\,\mathrm{tr}(\symbfit\Gamma_k^\top\symbfit \Sigma_k\symbfit\Gamma_k)
\right\}
=\symbfit 0
&\Longleftrightarrow
\symbfit \Sigma_k
\left\{
(\symbfit \Sigma_k+\lambda\symbfit I_{n_{\mathrm a}})\symbfit\Gamma_k
-
\symbfit Z
\right\}
=
\symbfit 0 .
\end{aligned}
\end{equation}
Since $\symbfit \Sigma_k$ is positive semidefinite and $\lambda>0$,
$\symbfit \Sigma_k+\lambda\symbfit I_{n_{\mathrm a}}$ is positive definite and hence invertible.
Therefore, we define
\begin{equation}
\label{eq:kernel-proof-Sk}
\symbfit S_k\coloneqq(\symbfit \Sigma_k+\lambda\symbfit I_{n_{\mathrm a}})^{-1}
\end{equation}
and consider the following candidate solution:
\begin{equation}
\label{eq:kernel-proof-gamma-star}
\symbfit\Gamma_k^\star=\symbfit S_k\symbfit Z .
\end{equation}
This candidate solution satisfies Eq.~\eqref{eq:kernel-proof-gamma-opt-cond}, as shown by Eqs.~\eqref{eq:kernel-proof-Sk} and~\eqref{eq:kernel-proof-gamma-star}:
\[
\begin{aligned}
&\symbfit \Sigma_k
\left\{
(\symbfit \Sigma_k+\lambda\symbfit I_{n_{\mathrm a}})
\symbfit\Gamma_k^\star
-
\symbfit Z
\right\}
\\
&=
\symbfit \Sigma_k
\left\{
(\symbfit \Sigma_k+\lambda\symbfit I_{n_{\mathrm a}})
\symbfit S_k\symbfit Z
-
\symbfit Z
\right\}
=
\symbfit \Sigma_k(\symbfit Z-\symbfit Z)
\quad(\text{by Eq.~\eqref{eq:kernel-proof-Sk}})
\\
&=
\symbfit 0 .
\end{aligned}
\]
Therefore, Eq.~\eqref{eq:kernel-proof-gamma-star} satisfies Eq.~\eqref{eq:kernel-proof-gamma-opt-cond}.
Moreover, because $\symbfit \Sigma_k$ is positive semidefinite and $\lambda>0$, the objective function in Eq.~\eqref{eq:kernel-proof-gamma-problem} is a differentiable convex function with respect to $\symbfit\Gamma_k$.
Therefore, $\symbfit\Gamma_k^\star$ satisfying the first-order condition in Eq.~\eqref{eq:kernel-proof-gamma-opt-cond} is a globally optimal coefficient matrix.
Consequently, by Eqs.~\eqref{eq:kernel-proof-representer-form} and~\eqref{eq:kernel-proof-gamma-star}, for any $\symbfit x\in\mathbb R^{1\times\tilde d(k)}$, we obtain
\begin{equation}
\label{eq:kernel-proof-gk-star}
g_k^\star(\symbfit x)
=\symbfit\kappa_k(\symbfit x)\symbfit S_k\symbfit Z.
\end{equation}
Next, we compute the optimal objective value for the fixed $\symbfit Z$ using this solution.
Since $\symbfit \Sigma_k$ is symmetric positive semidefinite, $\symbfit S_k$ is also symmetric.
From Eq.~\eqref{eq:kernel-proof-Sk}, the following identity holds:
\begin{equation}
\label{eq:kernel-proof-KS-identity}
\symbfit \Sigma_k\symbfit S_k-\symbfit I_{n_{\mathrm a}}
=
-\lambda\symbfit S_k .
\end{equation}
Therefore, by Eqs.~\eqref{eq:kernel-proof-gamma-star} and~\eqref{eq:kernel-proof-KS-identity}, we have
\begin{equation}
\label{eq:kernel-proof-residual}
\begin{aligned}
\symbfit \Sigma_k\symbfit\Gamma_k^\star-\symbfit Z
&=
(\symbfit \Sigma_k\symbfit S_k-\symbfit I_{n_{\mathrm a}})\symbfit Z
\quad(\text{by Eq.~\eqref{eq:kernel-proof-gamma-star}})
\\
&=
-\lambda\symbfit S_k\symbfit Z .
\end{aligned}
\end{equation}
Using Eq.~\eqref{eq:kernel-proof-residual}, the optimal objective value of Eq.~\eqref{eq:kernel-proof-gamma-problem} can be computed as follows:
\begin{equation}
\label{eq:kernel-proof-opt-value}
\begin{aligned}
&\|\symbfit \Sigma_k\symbfit\Gamma_k^\star-\symbfit Z\|_{\mathrm F}^2
+\lambda\,\mathrm{tr}({\symbfit\Gamma_k^\star}^\top
\symbfit \Sigma_k\symbfit\Gamma_k^\star)
\\
&=\lambda^2\,\mathrm{tr}(\symbfit Z^\top\symbfit S_k^2\symbfit Z)
+\lambda\,\mathrm{tr}(\symbfit Z^\top\symbfit S_k\symbfit \Sigma_k\symbfit S_k\symbfit Z)
\quad(\text{by Eqs.~\eqref{eq:kernel-proof-gamma-star} and~\eqref{eq:kernel-proof-residual}})
\\
&=\lambda\,\mathrm{tr}\!\left(\symbfit Z^\top\symbfit S_k
(\lambda\symbfit I_{n_{\mathrm a}}+\symbfit \Sigma_k)\symbfit S_k\symbfit Z\right)
\\
&=\lambda\,\mathrm{tr}(\symbfit Z^\top\symbfit S_k\symbfit Z)
\quad(\text{by Eq.~\eqref{eq:kernel-proof-Sk}}).
\end{aligned}
\end{equation}
Summing Eq.~\eqref{eq:kernel-proof-opt-value} over $k\in[K]$ yields
\begin{equation}
\label{eq:kernel-proof-total-opt-value}
\begin{aligned}
&\sum_{k=1}^K\left(
\|\symbfit \Sigma_k\symbfit\Gamma_k^\star-\symbfit Z\|_{\mathrm F}^2
+\lambda\,\mathrm{tr}({\symbfit\Gamma_k^\star}^\top
\symbfit \Sigma_k\symbfit\Gamma_k^\star)
\right)
\\
&=
\mathrm{tr}\!\left(\symbfit Z^\top
\lambda\sum_{k=1}^K\symbfit S_k
\symbfit Z\right)
=
\mathrm{tr}(\symbfit Z^\top\symbfit M_\lambda\symbfit Z).
\end{aligned}
\end{equation}
As a result, Eq.~\eqref{eq:kernel-formulation} reduces to the following problem with respect to $\symbfit Z$:
\[
\begin{aligned}
\min \quad
&\mathrm{tr}(\symbfit Z^\top\symbfit M_\lambda\symbfit Z)
\\[2mm]
\text{s.~t.} \quad
&\symbfit Z^\top\symbfit Z=\symbfit I_{\hat d},\\
&\symbfit{Z}\in\mathbb{R}^{n_{\mathrm{a}}\times\hat d}.
\end{aligned}
\]
This coincides with Eq.~\eqref{eq:reduced_KNI}.
The corresponding function $g_k^\star$ is given by Eq.~\eqref{eq:kernel-proof-gk-star}, which is identical to Eq.~\eqref{eq:gk-vector-solution}.
\end{proof}

\section{Additional Graph Constructions}
\label{app:graph-construction}

This appendix describes additional target-based graph constructions that are not used as the main graph choices in the experiments.
The symmetrization and party-wise aggregation procedures are the same as those in Eqs.~\eqref{eq:graph-weight-symmetrization} and~\eqref{eq:graph-weight-aggregation}.

\paragraph{Regression version of TSL}
For regression tasks, the target-similarity Laplacian (TSL) can be constructed using a smooth similarity based on target values:
\begin{equation}
\label{eq:tsl-reg-weight}
\hat w_{ii'}^{(Bk)}=
\begin{cases}
\exp\!\left(-\dfrac{(y_{\mathrm{a}i}-y_{\mathrm{a}i'})^2}{\sigma_y^2}\right)
& \text{if } i'\in\mathcal N_{k_{\mathrm{nn}}}^{(k)}(i),\\
0 & \text{otherwise,}
\end{cases}
\end{equation}
where $\sigma_y>0$ is a scale parameter for differences in target values.

\paragraph{Target-Dissimilarity Laplacian}
The target-dissimilarity Laplacian (TDL) is used as a penalty graph.
For classification tasks, TDL connects nearest-neighbor anchor instances with different target variables:
\begin{equation}
\label{eq:tdl-class-weight}
\hat w_{ii'}^{(Ck)}=
\begin{cases}
1 & \text{if } i'\in\mathcal N_{k_{\mathrm{nn}}}^{(k)}(i),\ y_{\mathrm{a}i}\neq y_{\mathrm{a}i'},\\
0 & \text{otherwise.}
\end{cases}
\end{equation}
For regression tasks, TDL can be constructed using a smooth dissimilarity based on target values:
\begin{equation}
\label{eq:tdl-reg-weight}
\hat w_{ii'}^{(Ck)}=
\begin{cases}
1-\exp\!\left(-\dfrac{(y_{\mathrm{a}i}-y_{\mathrm{a}i'})^2}{\sigma_y^2}\right)
& \text{if } i'\in\mathcal N_{k_{\mathrm{nn}}}^{(k)}(i),\\
0 & \text{otherwise.}
\end{cases}
\end{equation}

When TSL and TDL are used together, TSL defines the intrinsic-graph Laplacian $\symbfit B$, whereas TDL defines the penalty-graph Laplacian $\symbfit C$.
This combination encourages target-similar anchor instances to be close while maintaining dispersion with respect to target-dissimilar pairs.

\section{Additional Regression Evaluation}
\label{app:regression-evaluation}
To examine whether the proposed methods are also effective for regression tasks, we conducted additional experiments using UJIIndoorLoc~\cite{ujiindoorloc_310} and Relative location of CT slices on axial axis~\cite{slice_localization} (hereafter Slice localization) from the UCI Machine Learning Repository.
For UJIIndoorLoc, we used only latitude as the target variable and treated the task as a single-output regression problem.
For each dataset, we compared dimensionality reduction methods (PCA, Kernel PCA, and UMAP) and integration methods (KTI, KTI+GL, KTI+TSL, LTI, GEP, MPP, and ODC).
Except for the datasets, prediction target, and regression metric, the experimental conditions were the same as those in the main experiments.
As the regression model, we used \texttt{RandomForestRegressor} with default settings, corresponding to the Random Forest classifier used in the classification experiments.
Both explanatory and target variables were standardized to zero mean and unit variance before evaluation.
Accordingly, the scale parameter in the TSL weight in Eq.~\eqref{eq:tsl-reg-weight} was set to $\sigma_y^2=1$.

Each condition was repeated 20 times, and Table~\ref{tab:recon-score-2datasets-transposed-se} reports the mean RMSE and 95\% confidence interval.
In each cell, the mean value is shown in the upper row and the 95\% confidence interval is shown in parentheses in the lower row.
Bold values indicate the lowest RMSE among the integration methods shown in the table; ties are all shown in bold.
For reference, Local and Central achieved RMSEs of 13.02 $(\pm 0.08)$ and 6.41 $(\pm 0.07)$ on Slice localization, and 28.54 $(\pm 0.25)$ and 13.43 $(\pm 0.26)$ on UJIIndoorLoc, respectively.
In Table~\ref{tab:recon-score-2datasets-transposed-se}, K-PCA denotes Kernel PCA.

\begin{table}[t]
\centering
\scriptsize
\setlength{\tabcolsep}{4pt}
\renewcommand{\arraystretch}{1.15}
\caption{Regression results for the additional evaluation.}
\label{tab:recon-score-2datasets-transposed-se}
\begin{tabular*}{\linewidth}{@{\extracolsep{\fill}}llccccccc@{}}
\toprule
Dataset & DimRed
& KTI & KTI+GL & KTI+TSL & LTI & GEP & MPP & ODC \\
\midrule
Slice localization & PCA
& \valci{14.62}{0.40} & \valci{13.88}{0.56} & \valci{11.75}{0.41} & \valci{\textbf{8.91}}{0.25} & \valci{\textbf{8.91}}{0.25} & \valci{8.95}{0.23} & \valci{8.96}{0.21} \\
Slice localization & K-PCA
& \valci{9.76}{0.20} & \valci{9.73}{0.25} & \valci{\textbf{9.47}}{0.24} & \valci{9.92}{0.26} & \valci{9.92}{0.26} & \valci{9.77}{0.30} & \valci{9.83}{0.26} \\
Slice localization & UMAP
& \valci{12.18}{0.28} & \valci{\textbf{12.05}}{0.29} & \valci{12.10}{0.24} & \valci{12.90}{0.23} & \valci{12.91}{0.23} & \valci{13.01}{0.31} & \valci{12.89}{0.24} \\
UJIIndoorLoc & PCA
& \valci{27.93}{0.96} & \valci{25.11}{0.91} & \valci{24.47}{0.83} & \valci{\textbf{18.04}}{0.36} & \valci{\textbf{18.04}}{0.36} & \valci{18.11}{0.34} & \valci{\textbf{18.04}}{0.24} \\
UJIIndoorLoc & K-PCA
& \valci{\textbf{43.73}}{1.05} & \valci{44.63}{1.36} & \valci{43.77}{0.87} & \valci{45.05}{2.00} & \valci{45.05}{2.00} & \valci{50.15}{1.77} & \valci{46.49}{1.30} \\
UJIIndoorLoc & UMAP
& \valci{23.20}{0.53} & \valci{\textbf{22.85}}{0.48} & \valci{22.95}{0.52} & \valci{26.53}{0.63} & \valci{26.61}{0.66} & \valci{27.34}{0.76} & \valci{27.45}{0.76} \\
\bottomrule
\end{tabular*}
\end{table}

\FloatBarrier

As shown in Table~\ref{tab:recon-score-2datasets-transposed-se}, under nonlinear transformations using Kernel PCA and UMAP, KTI-based methods achieved RMSE comparable to or lower than that of linear integration methods (LTI, GEP, MPP, and ODC) in many cases.
In particular, under UMAP, KTI-based methods achieved lower RMSE than linear integration methods on both datasets.
For Slice localization under PCA, however, linear integration methods (LTI and GEP) substantially outperformed KTI, suggesting that KTI was most beneficial when the intermediate representations were nonlinear, whereas linear integration could remain preferable under linearly structured representations.
Overall, these results suggested that nonlinear integration could also be effective for regression tasks with nonlinear intermediate representations, which is consistent with the trend observed in the classification experiments.

\bibliographystyle{elsarticle-num}
\bibliography{cite}

@inproceedings{AGM,
  author    = {Balle, Borja and Wang, Yu-Xiang},
  title     = {Improving the {G}aussian Mechanism for Differential Privacy: Analytical Calibration and Optimal Denoising},
  booktitle = {Proceedings of the 35th International Conference on Machine Learning},
  series    = {Proceedings of Machine Learning Research},
  volume    = {80},
  pages     = {394--403},
  year      = {2018},
  publisher = {PMLR},
}

@inproceedings{Garbled-circuit2,
  author    = {Aner Ben-Efraim and Yehuda Lindell and Eran Omri},
  title     = {Optimizing Semi-Honest Secure Multiparty Computation for the Internet},
  booktitle = {Proceedings of the 2016 ACM SIGSAC Conference on Computer and Communications Security},
  pages     = {578--590},
  year      = {2016},
}

@inproceedings{DP,
  author    = {Cynthia Dwork},
  title     = {Differential Privacy: A Survey of Results},
  booktitle = {International Conference on Theory and Applications of Models of Computation},
  pages     = {1--19},
  publisher = {Springer},
  year      = {2008},
}

@misc{SecureComputationExample2,
  author    = {Adrià Gascon and Phillipp Schoppmann and Borja Balle and Mariana Raykova and Jack Doerner and Samee Zahur and David Evans},
  title     = {Privacy-Preserving Distributed Linear Regression on High-Dimensional Data},
  howpublished = {Cryptology ePrint Archive, Paper 2016/892},
  year      = {2016},
}

@book{MLP,
  author    = {Simon Haykin},
  title     = {Neural Networks: A Comprehensive Foundation},
  publisher = {Prentice Hall PTR},
  year      = {1994},
}

@article{DCframework1,
  author  = {Akira Imakura and Tetsuya Sakurai},
  title   = {Data Collaboration Analysis Framework Using Centralization of Individual Intermediate Representations for Distributed Data Sets},
  journal = {ASCE-ASME Journal of Risk and Uncertainty in Engineering Systems, Part A: Civil Engineering},
  volume  = {6},
  number  = {2},
  pages   = {04020018},
  year    = {2020},
}

@misc{FL2,
  author    = {Jakub Kone{\v{c}}n{\`y} and H Brendan McMahan and Felix X Yu and Peter Richt{\'a}rik and Ananda Theertha Suresh and Dave Bacon},
  title     = {Federated Learning: Strategies for Improving Communication Efficiency},
  howpublished = {arXiv preprint arXiv:1610.05492},
  year      = {2016},
}

@misc{SecureComputationExample3,
  author    = {Bita Darvish Rouhani and M. Sadegh Riazi and Farinaz Koushanfar},
  title     = {DeepSecure: Scalable Provably-Secure Deep Learning},
  howpublished = {Cryptology ePrint Archive, Paper 2017/502},
  year      = {2017},
}

@inproceedings{SecureComputation,
  author    = {Andrew C Yao},
  title     = {Protocols for Secure Computations},
  booktitle = {23rd Annual Symposium on Foundations of Computer Science (SFCS 1982)},
  pages     = {160--164},
  publisher = {IEEE},
  year      = {1982},
}

@article{HomomorphicEncryption2,
  author  = {Abbas Acar and Hidayet Aksu and A Selcuk Uluagac and Mauro Conti},
  title   = {A Survey on Homomorphic Encryption Schemes: Theory and Implementation},
  journal = {ACM Computing Surveys (CSUR)},
  volume  = {51},
  number  = {4},
  pages   = {1--35},
  year    = {2018},
}

@inproceedings{SecureAggregation,
  author    = {Keith Bonawitz and Vladimir Ivanov and Ben Kreuter and Antonio Marcedone and H Brendan McMahan and Sarvar Patel and Daniel Ramage and Aaron Segal and Karn Seth},
  title     = {Practical Secure Aggregation for Privacy-Preserving Machine Learning},
  booktitle = {Proceedings of the 2017 ACM SIGSAC Conference on Computer and Communications Security},
  pages     = {1175--1191},
  year      = {2017},
}

@misc{DCprivacy,
  author       = {Akira Imakura and Anna Bogdanova and Takaya Yamazoe and Kazumasa Omote and Tetsuya Sakurai},
  title        = {Accuracy and Privacy Evaluations of Collaborative Data Analysis},
  howpublished = {arXiv preprint arXiv:2101.11144},
  year         = {2021},
}

@article{NRI-DC,
  author  = {Akira Imakura and Tetsuya Sakurai and Yukihiko Okada and Tomoya Fujii and Teppei Sakamoto and Hiroyuki Abe},
  title   = {Non-Readily Identifiable Data Collaboration Analysis for Multiple Datasets Including Personal Information},
  journal = {Information Fusion},
  volume  = {98},
  pages   = {101826},
  year    = {2023},
}

@inproceedings{DCframework2,
  author    = {Akira Imakura and Xiucai Ye and Tetsuya Sakurai},
  title     = {Collaborative Data Analysis: Non-Model Sharing-Type Machine Learning for Distributed Data},
  booktitle = {Knowledge Management and Acquisition for Intelligent Systems},
  pages     = {14--29},
  publisher = {Springer},
  year      = {2021},
}

@article{FLSurvey2,
  author  = {Peter Kairouz and H Brendan McMahan and Brendan Avent and Aur{\'e}lien Bellet and Mehdi Bennis and Arjun Nitin Bhagoji and Kallista Bonawitz and Zachary Charles and Graham Cormode and Rachel Cummings and et~al.},
  title   = {Advances and Open Problems in Federated Learning},
  journal = {Foundations and Trends{\textregistered} in Machine Learning},
  volume  = {14},
  number  = {1--2},
  pages   = {1--210},
  year    = {2021},
}

@article{KawakamiDC,
    title = {New solutions based on the generalized eigenvalue problem for the data collaboration analysis},
    journal = {Information Sciences},
    volume = {723},
    pages = {122642},
    year = {2025},
    issn = {0020-0255},
    author = {Yuta Kawakami and Yuichi Takano and Akira Imakura}
}

@inproceedings{FLnon-iid,
  author    = {Qinbin Li and Yiqun Diao and Quan Chen and Bingsheng He},
  title     = {Federated Learning on Non-iid Data Silos: An Experimental Study},
  booktitle = {2022 IEEE 38th International Conference on Data Engineering (ICDE)},
  pages     = {965--978},
  publisher = {IEEE},
  year      = {2022},
}

@article{FLsurvey,
  author  = {Qinbin Li and Zeyi Wen and Zhaomin Wu and Sixu Hu and Naibo Wang and Yuan Li and Xu Liu and Bingsheng He},
  title   = {A Survey on Federated Learning Systems: Vision, Hype and Reality for Data Privacy and Protection},
  journal = {IEEE Transactions on Knowledge and Data Engineering},
  volume  = {35},
  number  = {4},
  pages   = {3347--3366},
  year    = {2023},
}

@article{FLChallenges,
  author  = {Tian Li and Anit Kumar Sahu and Ameet Talwalkar and Virginia Smith},
  title   = {Federated Learning: Challenges, Methods, and Future Directions},
  journal = {IEEE Signal Processing Magazine},
  volume  = {37},
  number  = {3},
  pages   = {50--60},
  year    = {2020},
}

@misc{FLnon-iid2,
  author    = {Xiang Li and Kaixuan Huang and Wenhao Yang and Shusen Wang and Zhihua Zhang},
  title     = {On the Convergence of FedAvg on Non-iid Data},
  howpublished = {arXiv preprint arXiv:1907.02189},
  year      = {2019},
}

@article{HomomorphicEncryption1,
  author  = {Paulo Martins and Leonel Sousa and Artur Mariano},
  title   = {A Survey on Fully Homomorphic Encryption: An Engineering Perspective},
  journal = {ACM Computing Surveys (CSUR)},
  volume  = {50},
  number  = {6},
  pages   = {1--33},
  year    = {2017},
}

@inproceedings{FL1,
  author    = {Brendan McMahan and Eider Moore and Daniel Ramage and Seth Hampson and Blaise Aguera y Arcas},
  title     = {Communication-Efficient Learning of Deep Networks from Decentralized Data},
  booktitle = {Artificial Intelligence and Statistics},
  pages     = {1273--1282},
  publisher = {PMLR},
  year      = {2017},
}

@article{sklearn,
  author  = {Fabian Pedregosa and Ga{\"e}l Varoquaux and Alexandre Gramfort and Vincent Michel and Bertrand Thirion and Olivier Grisel and Mathieu Blondel and Peter Prettenhofer and Ron Weiss and Vincent Dubourg and et~al.},
  title   = {Scikit-learn: Machine Learning in Python},
  journal = {Journal of Machine Learning Research},
  volume  = {12},
  pages   = {2825--2830},
  year    = {2011},
}

@inproceedings{Garbled-circuit1,
  author    = {Xiao Wang and Samuel Ranellucci and Jonathan Katz},
  title     = {Global-Scale Secure Multiparty Computation},
  booktitle = {Proceedings of the 2017 ACM SIGSAC Conference on Computer and Communications Security},
  pages     = {39--56},
  year      = {2017},
}

@misc{PPML,
  author    = {Runhua Xu and Nathalie Baracaldo and James Joshi},
  title     = {Privacy-Preserving Machine Learning: Methods, Challenges and Directions},
  howpublished = {arXiv preprint arXiv:2108.04417},
  year      = {2021},
}

@article{DC-DP,
  author  = {Hiromi Yamashiro and Kazumasa Omote and Akira Imakura and Tetsuya Sakurai},
  title   = {Toward the Application of Differential Privacy to Data Collaboration},
  journal = {IEEE Access},
  volume  = {12},
  pages   = {63292--63301},
  year    = {2024},
}

@inproceedings{LeakageFromGradients,
  author    = {Ligeng Zhu and Zhijian Liu and Song Han},
  title     = {Deep Leakage from Gradients},
  booktitle = {Advances in Neural Information Processing Systems},
  volume    = {32},
  pages     = {14774--14784},
  year      = {2019},
}

@article{MNIST,
  title={The mnist database of handwritten digit images for machine learning research [best of the web]},
  author={Deng, Li},
  journal={IEEE signal processing magazine},
  volume={29},
  number={6},
  pages={141--142},
  year={2012},
  publisher={IEEE}
}

@article{FASHION,
  title={{Fashion-MNIST}: a novel image dataset for benchmarking machine learning algorithms},
  author={Xiao, Han and Rasul, Kashif and Vollgraf, Roland},
  journal={arXiv preprint arXiv:1708.07747},
  year={2017}
}

@article{FedDCL,
  title={{FedDCL}: a federated data collaboration learning as a hybrid-type privacy-preserving framework based on federated learning and data collaboration},
  author={Imakura, Akira and Sakurai, Tetsuya},
  journal={arXiv preprint arXiv:2409.18356},
  year={2024}
}

@article{DC-COX,
  title={{DC-COX}: Data collaboration Cox proportional hazards model for privacy-preserving survival analysis on multiple parties},
  author={Imakura, Akira and Tsunoda, Ryoya and Kagawa, Rina and Yamagata, Kunihiro and Sakurai, Tetsuya},
  journal={Journal of Biomedical Informatics},
  volume={137},
  pages={104264},
  year={2023},
  publisher={Elsevier}
}

@article{DC-QE-med,
  title={Data collaboration for causal inference from limited medical testing and medication data},
  author={Nakayama, Tomoru and Kawamata, Yuji and Toyoda, Akihiro and Imakura, Akira and Kagawa, Rina and Sanuki, Masaru and Tsunoda, Ryoya and Yamagata, Kunihiro and Sakurai, Tetsuya and Okada, Yukihiko},
  journal={Scientific Reports},
  volume={15},
  number={1},
  pages={9827},
  year={2025},
  publisher={Nature Publishing Group UK London}
}

@article{DC-QE,
title = {Collaborative causal inference on distributed data},
journal = {Expert Systems with Applications},
volume = {244},
pages = {123024},
year = {2024},
issn = {0957-4174},
author = {Yuji Kawamata and Ryoki Motai and Yukihiko Okada and Akira Imakura and Tetsuya Sakurai}
}

@article{DC-SHAP,
  title={{DC-SHAP} method for consistent explainability in privacy-preserving distributed machine learning},
  author={Bogdanova, Anna and Imakura, Akira and Sakurai, Tetsuya},
  journal={Human-Centric Intelligent Systems},
  volume={3},
  number={3},
  pages={197--210},
  year={2023},
  publisher={Springer}
}

@article{DCapp5,
title = {Interpretable collaborative data analysis on distributed data},
journal = {Expert Systems with Applications},
volume = {177},
pages = {114891},
year = {2021},
issn = {0957-4174},
author = {Akira Imakura and Hiroaki Inaba and Yukihiko Okada and Tetsuya Sakurai}
}

@article{DC-SMOTE,
title = {Another use of SMOTE for interpretable data collaboration analysis},
journal = {Expert Systems with Applications},
volume = {228},
pages = {120385},
year = {2023},
issn = {0957-4174},
author = {Akira Imakura and Masateru Kihira and Yukihiko Okada and Tetsuya Sakurai}
}

@article{DCrecommender,
  author    = {Yanagi, Tomoya and Ikeda, Shunnosuke and Sukegawa, Noriyoshi and Takano, Yuichi},
  title     = {Privacy-preserving recommender system using the data collaboration analysis for distributed datasets},
  journal   = {PLoS ONE},
  year      = {2025},
  volume    = {20},
  number    = {4},
  pages     = {e0319954},
  publisher = {Public Library of Science}
}

@article{ky-fan,
  title = {On a theorem of {Weyl} concerning eigenvalues of linear transformations {I}},  author={Fan, Ky},
  journal={Proceedings of the National Academy of Sciences},
  volume={35},
  number={11},
  pages={652--655},
  year={1949}
}

@article{DC-odc,
title = {Data collaboration analysis with orthonormal basis selection and alignment},
journal = {Computers and Electrical Engineering},
volume = {135},
pages = {111192},
year = {2026},
issn = {0045-7906},
author = {Keiyu Nosaka and Yamato Suetake and Yuichi Takano and Akiko Yoshise}
}

@misc{UMAP,
    title={{UMAP}: Uniform Manifold Approximation and Projection for Dimension Reduction}, 
    author={Leland McInnes and John Healy and James Melville},
    year={2020},
    primaryClass={stat.ML},
}

@Article{Mashiko2025,
author={Mashiko, Sota
and Kawamata, Yuji
and Nakayama, Tomoru
and Sakurai, Tetsuya
and Okada, Yukihiko},
title={Anomaly detection in double-entry bookkeeping data by federated learning system with non-model sharing approach},
journal={Scientific Reports},
year={2025},
month={Nov},
day={26},
volume={15},
number={1},
pages={42208},
issn={2045-2322},
}

@article{LaplacianEigenmaps,
  title = {Laplacian Eigenmaps for Dimensionality Reduction and Data Representation},
  volume = {15},
  ISSN = {1530-888X},
  number = {6},
  journal = {Neural Computation},
  publisher = {MIT Press - Journals},
  author = {Belkin,  Mikhail and Niyogi,  Partha},
  year = {2003},
  month = jun,
  pages = {1373–1396}
}

@misc{DC-Clustering,
  title        = {A New Type of Federated Clustering: A Non-Model-Sharing Approach},
  author       = {Kawamata, Yuji and Kamijo, Kaoru and Kihira, Masateru and Toyoda, Akihiro and Nakayama, Tomoru and Imakura, Akira and Sakurai, Tetsuya and Okada, Yukihiko},
  year         = {2025},
  primaryClass = {cs.LG},
}

@article{need-protection1,
  title={When machine learning meets privacy: A survey and outlook},
  author={Liu, Bo and Ding, Ming and Shaham, Sina and Rahayu, Wenny and Farokhi, Farhad and Lin, Zihuai},
  journal={ACM Computing Surveys (CSUR)},
  volume={54},
  number={2},
  pages={1--36},
  year={2021},
  publisher={ACM New York, NY, USA}
}

@article{need-protection2,
author = {Fung, Benjamin C. M. and Wang, Ke and Chen, Rui and Yu, Philip S.},
title = {Privacy-preserving data publishing: A survey of recent developments},
year = {2010},
publisher = {Association for Computing Machinery},
address = {New York, NY, USA},
volume = {42},
number = {4},
issn = {0360-0300},
journal = {ACM Comput. Surv.},
month = jun,
articleno = {14},
numpages = {53}
}

@ARTICLE{PPDM,
  author={Mendes, Ricardo and Vilela, João P.},
  journal={IEEE Access}, 
  title={Privacy-Preserving Data Mining: Methods, Metrics, and Applications}, 
  year={2017},
  volume={5},
  number={},
  pages={10562-10582},
  }

@misc{DC-fed-17,
      title={Federated Learning System without Model Sharing through Integration of Dimensional Reduced Data Representations}, 
      author={Anna Bogdanova and Akie Nakai and Yukihiko Okada and Akira Imakura and Tetsuya Sakurai},
      year={2020},
      primaryClass={cs.LG},
}

@ARTICLE{GraphEmbeddingExtensions,
  author = {Shuicheng Yan and Dong Xu and B. Zhang and Hong-Jiang Zhang and Qiang Yang and Stephen Lin},  journal={IEEE Transactions on Pattern Analysis and Machine Intelligence}, 
  title={Graph Embedding and Extensions: A General Framework for Dimensionality Reduction}, 
  year={2007},
  volume={29},
  number={1},
  pages={40-51},
}

@ARTICLE{common_obfuscation,
  author={Kun Liu and Kargupta, H. and Ryan, J.},
  journal={IEEE Transactions on Knowledge and Data Engineering}, 
  title={Random projection-based multiplicative data perturbation for privacy preserving distributed data mining}, 
  year={2006},
  volume={18},
  number={1},
  pages={92-106},
  }

@inproceedings{private_obfuscation,
  title={On Lightweight Privacy-Preserving Collaborative Learning for Internet-of-Things Objects},
  author={Jiang, Linshan and Tan, Rui and Lou, Xin and Lin, Guosheng},
  booktitle={Proceedings of the International Conference on Internet of Things Design and Implementation},
  pages={70--81},
  year={2019}
}

@inproceedings{RandomProjection,
  author    = {Ella Bingham and Heikki Mannila},
  title     = {Random Projection in Dimensionality Reduction: Applications to Image and Text Data},
  booktitle = {Proceedings of the Seventh ACM SIGKDD International Conference on Knowledge Discovery and Data Mining},
  pages     = {245--250},
  year      = {2001},
}

@inproceedings{rotation_perturbation,
  author    = {Keke Chen and Ling Liu},
  title     = {Privacy Preserving Data Classification with Rotation Perturbation},
  booktitle = {Fifth IEEE International Conference on Data Mining (ICDM'05)},
  pages     = {589--592},
  year      = {2005},
}

@article{DP_Algorithmic,
author = {Dwork, Cynthia and Roth, Aaron},
title = {The Algorithmic Foundations of Differential Privacy},
year = {2014},
issue_date = {Aug 2014},
publisher = {Now Publishers Inc.},
address = {Hanover, MA, USA},
volume = {9},
number = {3–4},
issn = {1551-305X},
journal = {Found. Trends Theor. Comput. Sci.},
month = aug,
pages = {211–407},
numpages = {197}
}

@article{GDPR,
  author  = {Chris Jay Hoofnagle and Bart van der Sloot and Frederik Zuiderveen Borgesius},
  title   = {The {European Union} General Data Protection Regulation: What It Is and What It Means},
  journal = {Information \& Communications Technology Law},
  volume  = {28},
  number  = {1},
  pages   = {65--98},
  year    = {2019},
}

@Inbook{PPDM-attack,
    author="Liu, Kun
    and Giannella, Chris
    and Kargupta, Hillol",
    editor="Aggarwal, Charu C.
    and Yu, Philip S.",
    title="A Survey of Attack Techniques on Privacy-Preserving Data Perturbation Methods",
    bookTitle="Privacy-Preserving Data Mining: Models and Algorithms",
    year="2008",
    publisher="Springer US",
    address="Boston, MA",
    pages="359--381",
    isbn="978-0-387-70992-5",
}

@inproceedings{Generalized_Representer_Theorem,
  author    = {Bernhard Sch\"{o}lkopf and Ralf Herbrich and Alex J. Smola},
  title     = {A Generalized Representer Theorem},
  booktitle = {Proceedings of the 14th Annual Conference on Computational Learning Theory and 5th European Conference on Computational Learning Theory},
  pages     = {416--426},
  year      = {2001},
}

@misc{ujiindoorloc_310,
  author       = {Torres-Sospedra, Joaqu{\'i}n and Montoliu, Ra{\'u}l and Mart{\'i}nez-Us{\'o}, Adolfo and Arnau, Tom{\'a}s and Avariento, Joan},
  title        = {{UJIIndoorLoc}},
  year         = {2014},
  howpublished = {UCI Machine Learning Repository},
}

@misc{slice_localization,
  author       = {Graf, F. and Kriegel, H.-P. and Schubert, M. and P{\"o}lsterl, S. and Cavallaro, A.},
  title        = {{Relative location of CT slices on axial axis}},
  year         = {2011},
  howpublished = {UCI Machine Learning Repository},
}

\end{document}